\documentclass{article}

\usepackage{arxiv}

\usepackage[utf8]{inputenc} % allow utf-8 input
\usepackage[T1]{fontenc}    % use 8-bit T1 fonts
\usepackage{hyperref}       % hyperlinks
\usepackage{url}            % simple URL typesetting
\usepackage{booktabs}       % professional-quality tables
\usepackage{amsfonts}       % blackboard math symbols
\usepackage{amsmath,amsthm,amssymb}
\usepackage{nicefrac}       % compact symbols for 1/2, etc.
\usepackage{microtype}      % microtypography
\usepackage{cleveref}       % smart cross-referencing
\usepackage{lipsum}         % Can be removed after putting your text content
\usepackage{graphicx}
\usepackage[square,numbers]{natbib}
\usepackage{doi}

\usepackage{multirow,color}
\usepackage[small]{caption}
\usepackage{subcaption}
\usepackage{algorithm}
\usepackage[noend]{algpseudocode}
\usepackage{bm}
\usepackage{longtable}
\newtheorem{theorem}{Theorem}%[section]
\newtheorem{corollary}[theorem]{Corollary}
\newtheorem{lemma}[theorem]{Lemma}
\newtheorem{proposition}[theorem]{Proposition}

\theoremstyle{definition}
\newtheorem{definition}{Definition}
\newtheorem{assumption}[definition]{Assumption}
\theoremstyle{remark}

\numberwithin{equation}{section}

\newcommand{\paren} [1] {\ensuremath{ \left( {#1} \right) }}

\newcommand{\condparen}[2]{\ensuremath{\left(#1\left\lvert#2\right.\right)}}

\newcommand{\E}[1]{\mathbb{E}\left[#1\right]}
\newcommand{\prob}[1]{\text{Pr}\left\{#1\right\}}
\newcommand{\cond}{\,\bigg\vert\,}

\newcommand{\cardA}{\vert \mathcal{A} \vert}
\newcommand{\cardE}{\vert \mathcal{E} \vert}
\newcommand{\cardV}{\vert \mathcal{V} \vert}
\newcommand{\cardP}{\vert \mathcal{P} \vert}

\title{Online Learning of Energy Consumption for Navigation of Electric Vehicles\thanks{This paper is an extended version of \cite{ijcai2020-0284}.}}

\hypersetup{
pdftitle={Online Learning of Energy Consumption for Navigation of Electric Vehicles},
pdfsubject={cs.LG, stat.ML},
pdfauthor={Niklas {\AA}kerblom, Yuxin Chen, Morteza Haghir Chehreghani},
pdfkeywords={Energy efficient navigation, Online learning, Multi-armed bandits, Thompson sampling},
}

\begin{document}
%\tnotetext[mytitlenote]{This paper is an extended version of \cite{ijcai2020-0284}.}

\author{Niklas {\AA}kerblom \\
	Volvo Car Corporation \\
	Gothenburg, Sweden \\[3pt] % add vspace for next institution
	Department of Computer Science and Engineering\\
	Chalmers University of Technology\\
	Gothenburg, Sweden \\
	\texttt{niklas.akerblom@chalmers.se} \\
	%% examples of more authors
	\And
	Yuxin Chen \\
	Department of Computer Science\\
	University of Chicago\\
	Chicago, USA \\
	\texttt{chenyuxin@uchicago.edu} \\
    \And
	Morteza Haghir~Chehreghani \\
	Department of Computer Science and Engineering\\
	Chalmers University of Technology\\
	Gothenburg, Sweden \\
	\texttt{morteza.chehreghani@chalmers.se} \\
}

\maketitle

\begin{abstract}
Energy efficient navigation constitutes an important challenge in electric vehicles, due to their limited battery capacity. We employ a Bayesian approach to model the energy consumption at road segments for efficient navigation. In order to learn the model parameters, we develop an online learning framework and investigate several exploration strategies such as Thompson Sampling and Upper Confidence Bound. We then extend our online learning framework to the multi-agent setting, where multiple vehicles adaptively navigate and learn the parameters of the energy model. We analyze Thompson Sampling and establish rigorous regret bounds on its performance in the single-agent and multi-agent settings, through an analysis of the algorithm under batched feedback. Finally, we demonstrate the performance of our methods via experiments on several real-world city road networks.
\end{abstract}


\section{Introduction}
Today, electric vehicles experience a fast-growing role in many different transport systems. However, the applicability of electric vehicles is often constrained by the limited capacity of their batteries.
Due to the historically high cost of batteries, the range of electric vehicles has generally been much shorter than that of conventional vehicles. This has led to the fear of being stranded when the battery is depleted, an effect known as ``range anxiety''. Such concerns could be alleviated by improving the navigation algorithms and route planning methods for these systems. Therefore, in this paper we aim at developing principled methods for energy efficient navigation of electric vehicles.

Several works employ variants of shortest path algorithms for the purpose of finding the routes that minimize the energy consumption. Some of them (e.g., \cite{artmeier2010shortest,sachenbacher2011efficient}) focus on computational efficiency in searching for feasible paths where the constraints induced by limited battery capacity are satisfied. Both \cite{artmeier2010shortest} and \cite{sachenbacher2011efficient} use energy consumption as edge weights for the shortest path problem. They also consider recuperation of energy modeled as negative edge weights, since they identify that negative cycles cannot occur due to the law of conservation of energy. In \cite{sachenbacher2011efficient}, a consistent heuristic function for energy consumption is used with a modified version of A*-search \cite{hart1968formal} to capture battery constraints at query-time.
In \cite{baum_et_al:LIPIcs:2017:7608}, instead of using fixed scalar energy consumption edge weights, the authors use piecewise linear functions to represent the energy demand, as well as lower and upper limits on battery capacity. 

This task has also been developed beyond the shortest path problems in the context of the well-known vehicle routing problem (VRP). In  \cite{basso2019energy}, VRP is applied to electrified commercial vehicles in a two-stage approach (i.e., an electric vehicle routing problem, EVRP), where the first stage consists of finding the paths between customers with the lowest energy consumption and at the second stage the EVRP including optional public charging station nodes is solved. The same authors later extend their methods in \cite{basso2021electric}, using Bayesian methods to learn the energy consumption of individual road segments while solving the EVRP.

The aforementioned methods either assume that the necessary information for computing the optimal path is available, or do not provide enough exploration to acquire it. Thereby, we focus on developing an \emph{online} framework to learn (explore) the parameters of the energy model adaptively alongside solving the navigation (optimization) problem instances. 
We adopt a Bayesian approach to model the energy consumption for each road segment. The goal is to learn the parameters of such an energy model to be used for efficient navigation. Therefore, we develop an online learning framework to investigate and analyze several exploration strategies for learning the unknown parameters. 

Thompson Sampling (TS) \cite{thompson1933}, also called \emph{posterior sampling} and \emph{probability matching}, is a model-based exploration method for an optimal trade-off between exploration and exploitation. 
Several experimental \cite{GraepelCBH10,ChapelleL11,ChenRCK17} and theoretical studies \cite{agrawal2012analysis,KaufmannKM12,BubeckL13,osband2017posterior} have shown the effectiveness of Thompson Sampling in different settings. \cite{ChenRCK17} develops an online framework to explore the parameters of a decision model via Thompson Sampling in the application of interactive troubleshooting. 
In \cite{wang2018thompson}, Thompson Sampling is used for combinatorial semi-bandit problems, including the shortest path problem with Bernoulli-distributed edge costs, and distribution-dependent regret bounds are derived.

Upper Confidence Bound (UCB) \cite{Auer02} is another approach widely used for exploration-exploitation trade-off. A variant of UCB for combinatorial semi-bandits is introduced and analyzed in \cite{chen2013combinatorial}.
A Bayesian version of the Upper Confidence Bound method is introduced in \cite{kaufmann2012bayesian} and later analyzed in terms of regret bounds in \cite{kaufmann2018bayesian}. An alternative Bayesian approach is proposed in \cite{reverdy2014modeling}, which the authors call the Upper Credible Limit algorithm.

In this work, beyond the novel online learning framework for energy efficient navigation, we further extend our methods to the batched feedback and multi-agent settings. In the former, feedback from the environment is delayed and received in batches, while in the latter, multiple vehicles adaptively navigate and learn the parameters of the energy model. We then extensively analyze the proposed methods and evaluate them on several synthetic navigation tasks, as well as on real-world settings using SUMO-simulated traffic data from three different cities: Luxembourg \cite{codeca2017luxembourg}, Monaco \cite{MoSTCodeca2017} and Turin \cite{rapelli2021vehicular}. 

\subsection{Related Work}

The general problem considered in this paper is finding paths of minimum expected cost through graphs with unknown edge weight distributions. This problem has been studied using the framework of stochastic multi-armed bandits for at least a decade, where \cite{gai2012combinatorial} and \cite{liu2012adaptive} are prominent examples of early work addressing the problem. The authors of \cite{gai2012combinatorial} introduce a stochastic combinatorial bandit framework, where it is assumed that the weight of each edge in a traveled path is revealed afterwards (i.e., semi-bandit feedback). They propose a method called \textit{Learning with Linear Rewards} based on the celebrated principle of optimism in the face of uncertainty, where paths are selected using an exploration bonus added to the estimated mean of each edge. Other works, also based on the utilizing this principle for combinatorial bandits and online shortest path problems, are \cite{chen2013combinatorial} and \cite{zou2014online}.

Semi-bandit feedback is a natural assumption in our setting, since it is straightforward to record the actual energy consumed by a vehicle for each edge traversed. However, there are several methods for stochastic combinatorial bandits that do not need, nor utilize, this assumption. An example of such a method is \cite{liu2012adaptive}, in which the authors leverage path interdependencies using barycentric spanners. Other examples include any algorithm for the linear stochastic bandit model (see e.g., \cite{dani2008stochastic, abbasi2011improved}), of which combinatorial bandits with linear rewards is a special case.

Thompson Sampling has been analyzed and evaluated with promising results for combinatorial bandit problems in general (e.g., \cite{gopalan14, wen2015, wang2018thompson}), and the online shortest path problem is a commonly suggested application. The authors of \cite{russo2014learning} propose a framework for analyzing the Bayesian regret of Thompson Sampling, and apply it to several different problem settings. Their technique for converting regret bounds of UCB algorithms into bounds for Thompson Sampling is utilized in our work to derive bounds for batched feedback and multi-agent problem settings.

Bandit problems with delayed or batched feedback have been of intense interest to the research community, due to the wide applicability in real-world settings. Thompson Sampling has been empirically shown to achieve good results when reward observations are delayed \cite{ChapelleL11}. The authors of \cite{joulani2013online} propose a \textit{black box} algorithm which may convert any stochastic bandit algorithm into an algorithm handling delayed feedback. The converted algorithms retain the regret bounds of the original algorithms, except for an additive term which is constant in the horizon and linear in the maximum delay. In \cite{perchet2016batched}, a lower regret bound for the two-armed bandit problem with batched feedback is derived, again exhibiting a linear dependence with respect to the batch size. 

We take inspiration from the frequentist analysis of batched linear contextual UCB presented in \cite{han2020sequential} and extended to the generalized linear setting in \cite{ren2020batched}, utilizing a similar technique in our analysis to decompose the Bayesian regret over the batches. Another extension of \cite{han2020sequential} is \cite{ren2020dynamic}, which presents a greedy LASSO-based algorithm for a \textit{high-dimensional} batched linear contextual bandit setting, where the dimension of the context is assumed to be much higher than the time horizon. To provide an upper bound for the frequentist regret, they assume that the context is stochastic with enough variance to induce sufficient exploration. This assumption does not hold for the non-contextual setting studied in our work.

Finally, regarding incremental learning of energy consumption in graphs, the authors of \cite{basso2021electric} use a Bayesian approach, similar to the one in this work, to learn the edge-specific distributions of electric vehicle energy consumption in a road network. They utilize the posterior distributions to formulate and solve an EVRP for commercial vehicles, where the paths between customers, charging stations and depots are selected using learned parameters and information from the environment. Since exploration is not the focus of their work, their method of calculating the shortest paths most closely corresponds to the greedy baseline used for the experiments in this work.

\subsection{Our Contributions}

First and foremost, we propose a novel online learning framework for energy efficient navigation of electric vehicles, in a setting where the vehicle energy consumption of road segments is assumed to be stochastic and the corresponding distributions are unknown \textit{a priori}. We utilize a physical model of vehicle energy consumption to assign the edge-specific parameters of prior distributions for Bayesian bandit algorithms, such as Thompson Sampling and BayesUCB, in order to intelligently guide necessary exploration towards reasonable paths. The multi-armed bandit problem can be seen as a resource allocation problem, and as such, bandit algorithms are most useful where there is a limited number of agents available for data collection. 

While travel time in a road network is both stochastic and a common edge weight in shortest path problems, there is an abundance of travel time data available from various sources, e.g., from cellular devices. For vehicle energy consumption, however, there are factors limiting the number of agents. As energy consumption depends heavily on the specific vehicle type used, internal vehicle sensors are required for data collection. Furthermore, energy consumption also depends on the characteristics of the road traveled, like slopes, curvature, bumps, etc. Hence, it is a problem setting highly suited for Bayesian bandit algorithms.

While several works on Bayesian combinatorial bandit algorithms have been empirically evaluated using uninformative priors, it is less common with experiments where informative priors are used to explore combinatorial arm sets more efficiently. We not only utilize informative priors in our experiments, but also study the exploration of the road network through visual inspection of geospatial plots. Furthermore, we experimentally evaluate the robustness of the proposed framework to prior misspecification. We perform experiments for the road networks of multiple cities, using realistic traffic environment data.

As far as we are aware, there are no previous works analyzing the Bayesian regret of batched combinatorial Thompson Sampling. Furthermore, we also extend our analysis to the synchronous multi-agent setting. While there is prior work for batched linear contextual bandits (e.g., UCB in \cite{han2020sequential} and \cite{ren2020batched}), a combinatorial bandit problem is only a special case of the linear bandit problem for linear reward functions. Our technique, however, is feasible to extend for non-linear reward functions, such as in \cite{aakerblom2021online}, where combinatorial Thompson Sampling is used to address the problem of finding paths which minimize their maximum edge weights.  

Finally, this is the first work extending and evaluating the BayesUCB algorithm \cite{kaufmann2012bayesian} to the online shortest path problem, empirically demonstrating good performance of the algorithm in this problem setting.

\section{Energy Consumption Model}\label{sec:energy_model}

In this section, we start by describing how we model the road network and the different factors affecting the energy consumption of a vehicle traversing a specific road segment. We then outline two different Bayesian approaches to extend the deterministic energy consumption model to a probabilistic setting.

\subsection{Setup of the Energy Consumption Model}\label{sec:setup_energy_model}

We model the road network by a directed graph $\mathcal{G}(\mathcal{V}, \mathcal{E}, \bm{w})$ where each vertex $u \in \mathcal{V}$ represents an intersection of the road segments, and $\mathcal{E}$ indicates the set of directed edges. Each edge $e = (u_1, u_2) \in \mathcal{E}$ is an ordered pair of vertices $u_1, u_2 \in \mathcal{V}$ such that $u_1 \neq u_2$ and it represents the road segment between the intersections associated with $u_1$ and $u_2$. In the cases where bidirectional travel is allowed on a road segment represented by $(u_1,u_2) \in \mathcal{E}$, we add an edge $(u_2,u_1) \in \mathcal{E}$ in the opposite direction. A directed \emph{path} is a sequence of vertices $\langle u_1, u_2, \dots, u_n \rangle$, where $u_h \in \mathcal{V}$ for $h = 1, \dots, n$ and $(u_h, u_{h+1}) \in \mathcal{E}$ for $h = 1, \dots, n-1$. Hence, a path $\bm{p}$ can also be viewed as a sequence of edges. If $\bm{p}$ starts and ends with the same vertex, $\bm{p}$ is called a cycle. Note that, in this work, different paths may have different numbers of vertices.

We associate a weight vector $\bm{w}$ to the graph, where each element $w_e$ represents the total energy consumed by a vehicle traversing edge $e \in \mathcal{E}$. We extend the notation so that the total weight of a path $\bm{p}$ is denoted $w_{\bm{p}} := \sum_{e \in \bm{p}} w_e$. For each edge $e$, we also define other edge attributes associated with road segments, such as the average speed $v_e$, the length $l_e$, and the inclination $\alpha_e$.

In our setting, the amount of energy consumed at different road segments is stochastic and \textit{a priori} unknown. 
We adopt a Bayesian approach to model the energy consumption at each road segment $e\in \mathcal{E}$, i.e., the edge weights. Such a choice provides a principled way to induce prior knowledge. Furthermore, as we will see, this approach fits  well with the online learning and exploration of the parameters of the energy model. 

We first consider a deterministic model of vehicle energy consumption $E_e$ for an edge $e$, which will be used later as the prior. Similar to e.g., \cite{guzzella2007vehicle,basso2021electric}, our model is based on longitudinal vehicle dynamics and Newton's second law of motion. For convenience, we assume that vehicles drive with constant speed along individual edges so that we can disregard the longitudinal acceleration term. However, this assumption is only used for the prior. We then have the following equation for the approximated energy consumption (in watt-hours):
\begin{eqnarray}
    \label{eq:energy_consumption}
    E_e := \frac{m g l_e \sin (\alpha_e) + m g C_r l_e \cos (\alpha_e) + 0.5 C_d A \rho l_e v^2_e}{3600 \eta} .
\end{eqnarray}

In Eq. \ref{eq:energy_consumption} the vehicle mass $m$, the rolling resistance coefficient $C_r$, the front surface area $A$ and the air drag coefficient $C_d$ are vehicle-specific parameters. Whereas, the road segment length $l$, speed $v$ and inclination angle $\alpha$ are location (edge) dependent. In principle, $C_r$ could also be considered edge-specific (since it also depends on the surface of the road), but in this work, we assume that it is the same for all edges. We treat the gravitational acceleration $g$ and air density $\rho$ as constants. The powertrain efficiency $\eta$ is vehicle specific and can be approximated by a constant $\eta = 1$ for an ideal vehicle with no battery-to-wheel energy losses.

Actual energy consumption can be either positive (traction) or negative (regenerative braking). If the energy consumption is modeled accurately and used as $w_e$ in a graph $\mathcal{G}(\mathcal{V},\mathcal{E},\bm{w})$, the law of conservation of energy guarantees that there exists no cycle $\bm{c}$ in $\mathcal{G}$ where $w_{\bm{c}} < 0$. However, since we are modeling and estimating the expected energy consumption of each individual road segment independently (to ensure that the problem is tractable), this guarantee does not necessarily hold in our case. 

While modeling energy recuperation is desirable from an accuracy perspective, it introduces some difficulties. In terms of computational complexity, Dijkstra's algorithm \cite{dijkstra1959note} does not allow negative edge weights and the Bellman-Ford algorithm \cite{shimbel1955structure,ford1956network,bellman1958routing} is slower by an order of magnitude. There are methods to overcome this (e.g., \cite{johnson1977efficient}), but they still assume that there are no negative edge weight cycles in the network. Hence, we choose to only consider positive edge weights when solving the energy efficient (shortest path) problem, which enables us to use Dijkstra's algorithm in this work. This approximation still achieves meaningful results, since even with recuperation discarded, edges with high energy consumption are avoided.
So while the powertrain efficiency $\eta$ has a higher value when the energy consumption is negative than when it is positive, we believe using a constant is a justified simplification as we only consider positive edge-level energy consumption in the optimization stage. However, we emphasize that our generic online learning framework is independent of such approximations, and can be employed with any senseful energy model and shortest path algorithm.

\subsection{Rectified Gaussian Model of Energy Consumption} \label{sec:rectified_gaussian_model}

Motivated by \cite{wu2015electric}, as the first attempt at a probabilistic model of energy consumption, we assume the \emph{stochastic} energy consumption $\Tilde{E}_e$ of a road segment represented by an edge $e$ follows a Gaussian distribution, given a certain small range of inclination, vehicle speed and acceleration. We also assume that $\Tilde{E}_e$ is independent from $\Tilde{E}_{e'}$ for all $e' \in \mathcal{E}$ where $e' \neq e$ and that we may observe negative energy consumption. In other words, we assume that we may \textit{observe} the energy recuperation of the vehicle, even though we only use estimates of the non-negative energy consumption when solving the shortest path problem (as stated in Section \ref{sec:setup_energy_model}). The likelihood function (where, for later convenience, $\Tilde{E}_e$ is negated so that $\theta^*_e$ indicates a mean \emph{reward}) is then
\begin{equation*}
    P(\Tilde{E}_e \mid \theta^*_e, \sigma^2_e) := \mathcal{N}(-\Tilde{E}_e \mid \theta^*_e, \sigma^2_e) .
\end{equation*}

Here, for clarity, we assume the noise variance $\sigma_e^2$ is given.
We can then follow a Bayesian approach, and use a Gaussian conjugate prior over the mean energy consumption:
\begin{equation*}
    P(\theta^*_e \mid \mu_{e,0}, \varsigma^2_{e,0}) := \mathcal{N}(\theta^*_e \mid \mu_{e,0}, \varsigma^2_{e,0}) , 
\end{equation*}

where we choose $\mu_{e,0} \leftarrow -E_e$ and $\varsigma^2_{e,0} \leftarrow (\vartheta \mu_{e,0})^2$ for some constant $\vartheta > 0$. Due to the conjugacy properties, we have closed-form expressions for updating the posterior distributions with new observations of $\Tilde{E}_e$. For any path $\bm{p}$ in $\mathcal{G}$, we have $\E{\sum_{e \in \bm{p}} \Tilde{E}_e} = \sum_{e \in \bm{p}} \mathbb{E}[\Tilde{E}_e]$, which means we can find the path with the lowest expected energy demand if we set $w_e \leftarrow \mathbb{E}[\Tilde{E}_e]$ and solve the shortest path problem over $\mathcal{G}(\mathcal{V},\mathcal{E},\bm{w})$. 
When the expected energy consumption is estimated instead of being known, to deal with the risk of $w_e < 0$ (i.e., negative weights), we instead set $w_e \leftarrow \E{z_e}$ where $z_e$ is distributed according to the rectified Gaussian distribution $\mathcal{N}^R(-\theta^*_e, \sigma^2_e)$, which is defined so that $z_e := \max{(0, \Tilde{E}_e)}$ and $\Tilde{E}_e \sim \mathcal{N}(-\theta^*_e, \sigma^2_e)$. The expected value of $z_e$ is then $\E{z_e} = -(\theta^*_e \cdot (1 - \Phi(-\theta^*_e/\sigma_e)) + \sigma_e \cdot \phi(-\theta^*_e/\sigma_e))$, where $\Phi$ and $\phi$ are the standard Gaussian CDF and PDF respectively. Thus, since we observe both negative and positive energy consumption, we may utilize the conjugacy properties of the Gaussian likelihood and prior distribution to efficiently update and sample from the posterior distribution over the (non-negative) rectified Gaussian mean.

\subsection{Log-Gaussian Model of Energy Consumption}

Alternatively, instead of assuming a rectified Gaussian distribution for the energy consumption of each edge, we model the non-negative edge weights by (conjugate) Log-Gaussian likelihood and prior distributions. By definition, if we have a Log-Gaussian random variable $Z \sim \mathcal{LN}(\mu, \sigma^2)$, then the logarithm of $Z$ is a Gaussian random variable $(\log{Z}) \sim \mathcal{N}(\mu, \sigma^2)$. Therefore, we have the expected value $\mathbb{E}[Z] = \exp\{\mu + 0.5 \sigma^2\}$ and the variance $\mathbf{Var}[Z] = (\exp\{\sigma^2\} - 1)\cdot\exp\{2\mu + \sigma^2\}$. We can then define the likelihood function as
\begin{align}
    P\condparen{\Tilde{E}_e}{\theta^*_e, \sigma^2_e} := \mathcal{LN}\condparen{\Tilde{E}_e}{\log{(-\theta^*_e)} - \frac{1}{2} \log \paren{1 + \frac{\sigma^2_e}{\psi_e^2}},~\log \paren{1 + \frac{\sigma^2_e}{\psi_e^2}}} , \label{eq:log_normal_likelihood}
\end{align}
such that we match the moments of the rectified Gaussian model as well as possible, with $\mathbb{E}[\Tilde{E}_e] = -\theta^*_e$ and $\mathbf{Var}[\Tilde{E}_e] = \sigma^2_e \cdot \frac{(\theta^*_e)^2}{\psi_e^2}$. We also choose the prior hyper-parameters such that $\mathbb{E}[\theta^*_e] = \mu_{e,0}$ and $\mathbf{Var}[\theta^*_e] = \varsigma_{e,0}^2$, and also let $\psi_e = \mu_{e,0}$, where $\mu_{e,0}$ and $\varsigma_{e,0}$ are calculated in the same way as for the Gaussian prior (except that $\mu_{e,0}$ is restricted to be negative) in order to make fair comparisons between the Log-Gaussian and rectified Gaussian results. The resulting prior distribution is
\begin{align}
    P\big(\theta^*_e \big\vert \mu_{e,0}, \varsigma^2_{e,0}\big) :=  \mathcal{LN}\left(-\theta^*_e \big\vert \log{(-\mu_{e,0})} - \frac{1}{2} \log\paren{1 + \frac{\varsigma_{e,0}^2}{\mu_{e,0}^2}},~\log\paren{1 + \frac{\varsigma_{e,0}^2}{\mu_{e,0}^2}} \right) . \label{eq:log_normal_prior}
\end{align}

We emphasize that the specific parameterization that we use for the Log-Gaussian model in Eq. \ref{eq:log_normal_likelihood} allows for closed form posterior updates with the prior distribution in Eq. \ref{eq:log_normal_prior}. Since $- \theta_e$ is drawn from a Log-Gaussian prior distribution, then the value (i.e., a linear function of $\log (- \theta_e)$) of the first parameter of the Log-Gaussian likelihood described in Eq. \ref{eq:log_normal_likelihood} is Gaussian (i.e., the conjugate prior distribution of the first parameter using the standard parameterization). For more details on Bayesian updates with this Log-Gaussian parameterization, see e.g., \cite{herath2015conjugate}.
We summarize the notation used in the preceding sections and the rest of the paper in Table \ref{tbl:notation} of \ref{sec:notation}.

\section{Online Learning and Exploration of the Energy Model}\label{sec:online-single}
We develop an \emph{online learning} framework to explore the parameters of the energy model adaptively alongside sequentially solving the navigation (optimization) problem at different time steps. Here, a \textit{time step} (or round) refers to each time we select and traverse a path.
At the beginning, the exact energy consumption of the road segments and the parameters of the respective model are unknown. Thus, we start with an approximate and possibly inaccurate estimate of the parameters. We use the current estimates to solve the current navigation task. We then update the model parameters according to the observed energy consumption at different road segments (edges) of the navigated path, and use the new parameters to solve the next task.

Algorithm \ref{alg:online_algorithm} describes these steps, where the vectors $\bm{\mu}_{t-1}$ and $\bm{\varsigma}_{t-1}$ refer to the current posterior parameters of the energy model for all the edges at the current time $t$, which are used to obtain the current edge weight vector $\bm{w}_t$. Whenever we refer to an element of a vector indexed by a time step $t$, we always let the \emph{rightmost} index be $t$, e.g., $w_{e,t}$ in the vector $\bm{w}_t$. We solve the optimization problem using $\bm{w}_t$ to determine the optimal action (or \emph{arm} in the nomenclature of multi-armed bandit problems) $\bm{a}_t$, which in this context is a path through a graph. The action $\bm{a}_t$ is applied and a reward $r_t(\bm{a}_t)$ is observed, consisting of the actual measured energy consumption for each of the passed edges. We assume that the energy consumption distribution of each edge is fixed over time, and therefore, we exclude the subscript $t$ of the reward where it is not needed, such as for the expected reward $\E{r (\bm{a}_t)}$.
Since we want to minimize energy consumption, we regard it as a negative reward when we update the parameters (shown for example for the rectified Gaussian model in Algorithm \ref{alg:gaussian_update_parameters}). $T$ indicates the total number of time steps, sometimes called the horizon. 

To measure the effectiveness of our online learning algorithm, we consider its regret, which is the difference in the total expected reward between always playing the optimal action and playing actions according to the algorithm. Formally, the instant regret at time $t$ (or alternatively the \emph{gap} of the action selected at time $t$) is defined as $\Delta_t := \E{r (\bm{a}^*)} - \E{r (\bm{a}_t)}$ where $\bm{a}^* := \arg \max_{\bm{a}} \E{r (\bm{a})}$ is the action which maximizes the expected reward, and the cumulative regret is defined as $\text{Regret}(T) := \sum_{t=1}^T \Delta_t$. Since our framework uses a Bayesian approach, we also consider \emph{Bayesian regret}, which is the expected value of the regret over problem instances sampled from the prior distribution, so that $\text{BayesRegret}(T) := \mathbb{E} \left[\text{Regret}(T)\right]$.

\begin{algorithm}[tb]
\caption{Online learning for energy efficient navigation}
\label{alg:online_algorithm}
\begin{algorithmic}[1]
\Require $\bm{\mu}_0, \bm{\varsigma}_0$
\For{$t \leftarrow 1, \dots, T$}
    \label{row:optimization_objective}\State $\bm{w}_t \leftarrow $ \Call{GetEdgeWeights}{$t, \bm{\mu}_{t-1}, \bm{\varsigma}_{t-1}$}
    \label{row:optimization_solution}\State $\bm{a}_t \leftarrow $ \Call{SolveOptimizationToFindAction}{$\bm{w}_t$}
    \label{row:apply_action}\State $\bm{r}_t \leftarrow $ \Call{ApplyActionAndObserveReward}{$\bm{a}_t$}
    \label{row:update_parameters}\State $\bm{\mu}_{t}, \bm{\varsigma}_{t} \leftarrow $ \Call{UpdateParameters}{$\bm{a}_t, \bm{r}_t, \bm{\mu}_{t-1}, \bm{\varsigma}_{t-1}$}
\EndFor
\end{algorithmic}
\end{algorithm}

\begin{algorithm}[tb]
\caption{Gaussian parameter update of the energy model}
\label{alg:gaussian_update_parameters}
\begin{algorithmic}[1]
\Procedure{UpdateParameters}{$\bm{a}_t, \bm{r}_t, \bm{\mu}_{t-1}, \bm{\varsigma}_{t-1}$}
\For{each edge $e \in \bm{a}_t$}
\State $\varsigma_{e,t}^2 \leftarrow \left(\frac{1}{\varsigma_{e,t-1}^2} + \frac{1}{\sigma_{e}^2}\right)^{-1}$
\State $\mu_{e,t} \leftarrow \varsigma_{e,t}^2 \left( \frac{\mu_{e,t-1}}{\varsigma_{e,t-1}^2} + \frac{r_t(e)}{\sigma^2_e}\right)$
\EndFor
\State \Return $\bm{\mu}_{t}, \bm{\varsigma}_{t}$
\EndProcedure
\end{algorithmic}
\end{algorithm}

\subsection{Shortest Path Problem as Multi-Armed Bandit}\label{sec:cmab_formulation}

A combinatorial bandit \cite{cesa2012combinatorial,gai2012combinatorial} is a multi-armed bandit problem where an agent is only allowed to pull sets of arms instead of an individual arm. However, there may be restrictions on the feasible combinations of the arms. We consider the combinatorial semi-bandit case where the rewards are observed for each individual arm pulled by an agent during a round. 

A number of different combinatorial problems can cast to multi-armed bandits in this way, among them the online shortest path problem \cite{gai2012combinatorial,liu2012adaptive,zou2014online} is the focus of this work. An efficient algorithm for the deterministic problem (e.g., \cite{dijkstra1959note}) can be used as an oracle \cite{wang2018thompson} to provide feasible sets of arms to the agent, as well as to maximize the expected reward.

We connect this to the optimization problem in Algorithm \ref{alg:online_algorithm}, where we want to find an arm $\bm{a}_t$. At time $t$, let $\mathcal{G}(\mathcal{V},\mathcal{E},\bm{w}_t)$ be a directed graph with weight vector $\bm{w}_t$ and sets of vertices $\mathcal{V}$ and edges $\mathcal{E}$. Given a source vertex $u_1 \in \mathcal{V}$ and a target vertex $u_n \in \mathcal{V}$, let $\mathcal{P}$ be the set of all paths $\bm{p}$ in $\mathcal{G}$ such that $\bm{p} = \langle u_1, \dots, u_n \rangle$. Assuming non-negative edge costs $w_{e,t}$ for each edge $e \in \mathcal{E}$, the problem of finding the shortest path (arm $\bm{a}_t$) from $u_1$ to $u_n$ can be formulated as 
\begin{equation*}
    \bm{a}_t = \arg\,\min\limits_{\bm{p} \in \mathcal{P}}\,\sum_{e\in \bm{p}}w_{e,t} .
\end{equation*}

For the analysis, we introduce some formal definitions for this stochastic combinatorial semi-bandit problem. There is a set of \emph{base arms} $\mathcal{A}$, which corresponds to $\mathcal{E}$ in the considered graph. The set of arms selected at time $t$ is called the \emph{super-arm} $\bm{a}_t \subseteq \mathcal{A}$. The set of feasible super-arms $\mathcal{I}$ such that $\bm{a}_t \in \mathcal{I}$, is equal to the set of paths $\mathcal{P}$. We further define the expected reward of super-arm $\bm{a}$ with respect to a particular mean reward vector (for all base arms) $\bm{\theta}$ as $f_{\bm{\theta}} (\bm{a}) := \sum_{i \in \bm{a}} \theta_i$. Hence, according to the previously introduced definition of regret, we have that $\text{Regret}(T) = \sum_{t=1}^T \left(f_{\bm{\theta}^*} (\bm{a}^*) - f_{\bm{\theta}^*} (\bm{a}_t) \right)$.

\subsection{Thompson Sampling}\label{sec:alg:ts}

In our Bayesian setup, a greedy strategy chooses the arm which maximizes the expected reward according to the current estimate of the mean rewards. Since the greedy method does not actively explore the environment, there are other methods which perform better in terms of minimizing cumulative regret. One commonly used method is $\epsilon$-greedy, where a (uniformly) random arm is taken with probability $\epsilon$ and the greedy strategy is used otherwise. While, in principle, it could possible to select paths uniformly at random for the exploration time steps, the size of the set of all paths $\mathcal{P}$ (corresponding to the set of feasible super-arms $\mathcal{I})$ can be exponential with respect to the number edges in the graph. This might even include paths similar to random walks through the graph, which would almost certainly be very inefficient in terms of accumulated edge costs. Hence, this method is not well suited to the shortest path problem. A modification of $\epsilon$-greedy (based on Algorithm 1 in the supplementary material of \cite{chen2013combinatorial}), where only a single edge (and the shortest path through it) is sampled, is introduced in Algorithm \ref{alg:epsilon_greedy_cmab}. However, for large graphs this might still lead to unreasonable exploration paths (e.g., a path between New York City and Boston through a randomly selected detour around Los Angeles).

An alternative method for exploration is Thompson Sampling (TS). In contrast to the greedy method, with TS (like in $\epsilon$-greedy), arms are randomly sampled. However, where arms are sampled uniformly at random with $\epsilon$-greedy, the TS agent samples from the model, i.e., during each time step, it selects an arm which has a high probability of being optimal by sampling mean rewards from the posterior distribution and choosing an arm which maximizes them. In other words, the method utilizes the prior beliefs about the parameter values to guide exploration towards reasonable arms.

Thompson Sampling for the energy consumption shortest path problem is outlined in Algorithm \ref{alg:ts_optimization_objective}, where it can be used in Algorithm \ref{alg:online_algorithm} to obtain the edge weights in the network (only shown for the rectified Gaussian model). 

\subsubsection{Regret analysis}\label{sec:ts_regret_analysis}

In the following section, we provide an analysis on the cumulative regret of Thompson Sampling for the shortest path navigation problem. While better upper bounds on Bayesian regret for combinatorial TS is possible (e.g., using our proof for the batched combinatorial setting in Theorem \ref{thm:b_ts_cmab_regret} with batch size 1, we obtain a Bayesian regret upper bound of $\Tilde{\mathcal{O}} \left(\cardE \sqrt{T}\right)$), this result may give some insight on the relationship between reinforcement learning problems and combinatorial bandit problems.

\begin{proposition}
\label{prop:graph_regret_bound_ts}
The Bayesian regret of Algorithm \ref{alg:online_algorithm} is upper bounded by
\begin{align*}
    \text{BayesRegret}(T) \leq \Tilde{\mathcal{O}}\left(\cardV^2 \sqrt{\cardE \; T}\right) .
\end{align*}
\end{proposition}

We arrive at this result by relating the problem to recent results in reinforcement learning literature \cite{osband2017posterior}. We view the online shortest path problem as an episodic reinforcement learning problem on an unknown finite time horizon Markov decision process (MDP). Here, each vertex $u \in \mathcal{V}$ corresponds to a state, each edge $e \in \mathcal{E}$ corresponds to an action, and the reward distributions for each action are the same as in the bandit problem. Like in the bandit problem formulation, the rewards of different states are assumed to be independent. Furthermore, given a state and an (allowed) action, transitions are deterministic, such that the next state is the end vertex of the edge corresponding to the action. Each episode starts in the source vertex state and ends when the target vertex state is reached. In other words, each episode corresponds a time step (and path selection) in the bandit problem formulation.

Applying posterior sampling for reinforcement learning (PSRL) like in \cite{OsbandRR13}, to this problem, using identical priors over reward distribution parameters as in the bandit problem, is equivalent to using TS on the combinatorial semi-bandit problem. At the start of each episode, PSRL samples an MDP from the current prior / posterior distribution over MDPs (here, a distribution over reward distributions, since the transitions are deterministic and known).

The policy used during this episode by PSRL is then the optimal policy with respect to the sampled MDP. In this problem, since the rewards are the negative edge weights of the graph, the shortest path between the source and target vertices will be selected.

Since the posterior parameters involved in PSRL are updated in the same way as in the bandit problem, identical observations and samples will lead to identical posterior updates. Hence, they are equivalent, and a regret bound for one will apply to the other. From \cite{osband2017posterior}, with $\tau$ being the episode length and $T$ from the bandit problem corresponding to the number of episodes in the RL problem, we have
\begin{align*}
    \text{BayesRegret}(T) &\leq \Tilde{\mathcal{O}} \left( \tau \sqrt{\cardV \; \cardE \; \tau T} \right) \\
    &\leq \Tilde{\mathcal{O}}\left(\cardV^2 \sqrt{\cardE \; T}\right).
\end{align*}

We also note that Conjecture 1 of \cite{osband2017posterior} would improve this result so that
\begin{align*}
    \text{BayesRegret}(T) &\leq \Tilde{\mathcal{O}}\left(\cardV \sqrt{\cardV \; \cardE \; T}\right).
\end{align*}

\begin{algorithm}[tb]
\caption{\label{alg:ts_optimization_objective}Thompson Sampling}
\begin{algorithmic}[1]
\Procedure{GetEdgeWeights}{$t, \bm{\mu}_{t-1}, \bm{\varsigma}_{t-1}$}
\For{each edge $e \in \mathcal{E}$}
\State $\tilde{\theta}_e \leftarrow$ Sample from posterior
$\mathcal{N}(\mu_{e,t-1}, \varsigma^2_{e,t-1})$
\State $w_{e,t} \leftarrow \mathbb{E}[z_e]$ where $z_e \sim \mathcal{N}^R(-\tilde{\theta}_e, \sigma^2_e)$ 
\EndFor
\State \Return $\bm{w}_t$
\EndProcedure
\end{algorithmic}
\end{algorithm}

We note that the combinatorial semi-bandit problem formulation of Section \ref{sec:cmab_formulation} can be seen as a simpler special case of the reinforcement learning problem with less complexities to learn (e.g., less parameters to estimate, no state transitions modeling, etc.). In particular, whereas the traffic environment is affected by the paths that we choose, any state changes caused by an agent do not typically affect it immediately, since an edge is likely not traversed more than once during a single episode (path). If we want to adapt to the observed immediate rewards of different base arms while driving on a selected path, this could be modeled as a reinforcement learning problem, e.g., like the (\#P-hard) \textit{stochastic shortest path problem with recourse} \cite{polychronopoulos1996recourse}, which may (in principle) then be addressed by PSRL. In general, however, choosing the less complex (though still meaningful) bandit problem formulation enables us to use powerful methods with proven strong performance guarantees.

\subsection{Bayesian Upper Confidence Bound}
Another class of algorithms demonstrated to work well in the context of multi-armed bandits is the collection of the methods developed around the Upper Confidence Bound (UCB). Informally, these methods are designed based on the principle of optimism in the face of uncertainty. The algorithms achieve efficient exploration by choosing the arm with the highest empirical mean reward added to an exploration term (the confidence width). Hence, the arms chosen are those with a plausible possibility of being optimal.

In \cite{chen2013combinatorial} a combinatorial version of UCB (CUCB) is shown to achieve sub-linear regret for combinatorial semi-bandits. However, using a Bayesian approach is beneficial in this problem since it allows us to employ the  theoretical knowledge on the energy consumption in a prior. Hence, we consider BayesUCB \cite{kaufmann2012bayesian} and adapt it to the combinatorial semi-bandit setting. Similar to \cite{kaufmann2012bayesian}, we denote the quantile function for a distribution $\lambda$ as $Q(\beta,\lambda)$, defined such that for a random variable distributed according to $\lambda$ (s.t. $X \sim \lambda$), we have $\text{Pr}(X \leq Q(\beta, \lambda)) = \beta $. The idea of that work is to use upper quantiles of the posterior distributions of the expected arm rewards to select arms. If $\lambda$ denotes the posterior distribution of a base arm and $t$ is the current time step, the Bayesian Upper Confidence Bound (BayesUCB) for that base arm is $Q(1 - 1/t, \lambda)$.

This method is outlined in Algorithm \ref{alg:ucb_algorithm} for the rectified Gaussian model. Here, since the goal is to minimize the energy consumption which can be considered as the negative of the reward, we use the \emph{lower} quantile $Q(1/t, \lambda)$.

\begin{algorithm}[tb]
\caption{\label{alg:ucb_algorithm}BayesUCB}
\begin{algorithmic}[1]
\Procedure{GetEdgeWeights}{$t, \bm{\mu}_{t-1}, \bm{\varsigma}_{t-1}$}
\For{each edge $e \in \mathcal{E}$}
\State $-\tilde{\theta}_e \leftarrow Q\left(\frac{1}{t}, \mathcal{N}(-\mu_{e,{t-1}},\varsigma_{e,{t-1}}^2)\right)$
\State $w_{e,t} \leftarrow \mathbb{E}[z_e]$ where $z_e \sim \mathcal{N}^R(-\tilde{\theta}_e, \sigma^2_e)$ 
\EndFor
\State \Return $\bm{w}_t$
\EndProcedure
\end{algorithmic}
\end{algorithm}

\section{Multi-Agent Learning and Exploration}
\label{sec:multi_agent}

The online learning may speed up via having multiple agents exploring simultaneously and sharing information on the observed rewards with each other. In our particular application, this corresponds to a fleet of vehicles of similar type sharing information about energy consumption across the fleet. Such a setting can be very important for road planning, electric vehicle industries, vehicle fleet operators and city principals. 

The communication between the agents for the sake of sharing the the observed rewards can be synchronous or asynchronous. In this paper, we consider the synchronous setting, where the vehicles drive concurrently in each time step and share their accumulated knowledge with the fleet before the next iteration starts. At each time step, each individual vehicle independently selects a path to explore/exploit according to the online learning strategies provided in Section \ref{sec:online-single}. Here, we assume that all vehicles start their paths with the same source vertex and end them at the same target vertex, though even without this assumption, vehicles would benefit from information sharing as long as there is some overlap between selected paths. The vehicles share information synchronously, when all agents have finished their trips for a certain time step. During each time step, the agents are allowed to select paths which are overlapping (with shared edges), but we do not model any physical interactions between vehicles (e.g., how increased traffic intensity on those road segments affects energy consumption). However, this could be an interesting topic for future work.

Below, we provide two different regret bounds for TS-based multi-agent learning under the synchronous setting. Both are based on the idea of viewing the synchronous multi-agent problem as a single-agent problem with delayed feedback received in batches. Specifically, the \textit{delay} corresponds to the number of vehicles in the fleet, since we wait for all of them to finish traversing their selected paths until we update the posterior distributions and start the next time step.

\subsection{Thompson Sampling with Queued Delayed Feedback}

The first approach is based on the method of \cite{joulani2013online}, which converts any algorithm for non-delayed stochastic bandit problems to an algorithm which handles delayed feedback, with a term constant in $T$ added to the regret. This method and other similar queue-based methods have previously been used to adapt (and analyze) existing bandit algorithms for various problem settings with delayed feedback (see e.g., \cite{mandel2015queuemethod, huang2021delayedcellular}). The approach of \cite{joulani2013online} is to wrap the original algorithm in an outer algorithm, which they call Queued Partial Monitoring with Delays (QPM-D). In essence, the inner algorithm functions as in the non-delayed case, unless the feedback of a selected arm is delayed and not available yet. In that case, the outer algorithm takes over and repeatedly plays the selected arm until feedback is received. Since the arm is played multiple times, excess delayed feedback, not immediately used by the inner algorithm, is also received. The outer algorithm stores the excess feedback in a queue data structure (where the order in which elements are inserted is also the order in which they are later retrieved, i.e., \textit{First In, First Out}, or FIFO). This allows the inner algorithm to retrieve feedback from the queue the next time the arm is selected, instead of having to wait for delayed feedback. We outline QPM-D adapted to our problem in Algorithm \ref{alg:qpm_d}.

\begin{center}
\begin{algorithm}[H]
\caption{QPM-D for Algorithm \ref{alg:online_algorithm}}
\label{alg:qpm_d}
\begin{algorithmic}[1]
\State Create an empty $\text{Queue}[\bm{a}]$ for each $\bm{a} \in \mathcal{I}$.
\State Let $\bm{b} \in \mathcal{I}$ be the first super-arm selected by Algorithm \ref{alg:online_algorithm}.
\For{$t \leftarrow 1, \dots, T$}
    \State \textbf{Predict:}
    \While{$\text{Queue}[\bm{b}]$ is non-empty}
        \State Update Algorithm \ref{alg:online_algorithm} with one reward from $\text{Queue}[\bm{b}]$.
        \State Let $\bm{b}$ be the next super-arm selected by Algorithm \ref{alg:online_algorithm}.
    \EndWhile
    \State There are no queued rewards for $\bm{b}$, so perform arm $\bm{a}_t \leftarrow \bm{b}$ at time $t$ to receive rewards (possibly delayed) by the environment.
    \State \textbf{Update:}
    \State Let $\mathcal{D}_t$ be the set of (delayed) rewards received at time $t$ and each $(s, r_s(\bm{a}_s)) \in \mathcal{D}_t$ be the timestamped reward $r_s(\bm{a}_s)$ resulting from the arm $\bm{a}_s$ at time $s$.
    \For{$(s, r_s(\bm{a}_s)) \in \mathcal{D}_t$}
        \State Add the reward $r_s(\bm{a}_s)$ to $\text{Queue}[\bm{a}_s]$. 
    \EndFor
\EndFor
\end{algorithmic}
\end{algorithm}
\end{center}

\begin{theorem}\label{thm:ma_regret_bound_ts} Let $K$ be the number of agents, $T$ be the horizon and $\text{Regret}_k (T)$ be the regret of each agent $k \in [K]$. In the synchronous multi-agent online shortest path setting (i.e., a fleet of $K$ agents / vehicles working in parallel in each time step), the total fleet regret incurred by invoking Algorithm \ref{alg:qpm_d} satisfies
$\sum_{k=1}^K \text{Regret}_k (T) \leq \mathcal{O}\left(\cardP K + \text{Regret} (TK)\right)$.
\end{theorem}

\begin{proof}[Proof]
The result is obtained as a corollary of Theorem 6 in \cite{joulani2013online} which converts online algorithms for the non-delayed case to ones that can handle delays in the feedback (i.e., Algorithm \ref{alg:qpm_d}), while retaining their theoretical guarantees. We consider the online shortest path problem as a standard stochastic bandit problem where the paths are the arms, and handle the multi-agent setting using Algorithm \ref{alg:qpm_d}, like a sequential setting with delayed feedback. Let $\kappa_t$ denote the feedback delay of the action at time $t$. Then according to \cite{joulani2013online} we have
\begin{align*}
    \sum_{k=1}^K \text{Regret}_k (T) &\leq \text{Regret} (TK) + \sum_{\bm{p} \in \mathcal{P}} \mathcal{O}\left(\max_t \kappa_t\right) \\
    &\leq \text{Regret} (TK) + \sum_{\bm{p} \in \mathcal{P}} \mathcal{O}\left(K\right) \\
    &\leq \mathcal{O}\left(\cardP K + \text{Regret} (TK)\right) .
\end{align*}
\end{proof}

While the additional first term of the regret is constant in $T$, it is also linear in $\cardP$, which may be exponential w.r.t. $\cardE$. 

\subsection{Thompson Sampling with Batched Feedback}

In order to remove the exponential factor in Theorem \ref{thm:ma_regret_bound_ts}, we outline a second approach. While the synchronous multi-agent setting \emph{can} be cast as a delayed feedback problem, the general delay model is not actually necessary. Since the updates are synchronous, viewing it as a \emph{batched} problem setting is sufficient. In this setting, rewards for selected arms are received periodically at fixed intervals, i.e., like \emph{tumbling windows}. We note that this problem formulation can be useful beyond the multi-agent setting, e.g., in environments where feedback may be delayed due to wireless connection problems.

The regret analysis is not as straightforward as the one for Theorem \ref{thm:ma_regret_bound_ts}. We combine ideas on batched bandit algorithms and analyses from \cite{gao2019batched}, \cite{han2020sequential} and \cite{ren2020batched} with the general proof framework for deriving Bayesian regret bounds introduced by \cite{russo2014learning}. Before considering the multi-agent case, we start by outlining Thompson Sampling for the batched combinatorial semi-bandit setting in Algorithm \ref{alg:b_ts_cmab}. Here, we first consider a general stochastic combinatorial semi-bandit problem (i.e., not limited to the online shortest path problem) where rewards for each base arm $i \in \mathcal{A}$ are drawn from $\mathcal{N}\left(\theta^*_i, \sigma_i^2\right)$, with $\theta_i^* \sim \mathcal{N}\left(\mu_{i, 0}, \varsigma^2_{i,0}\right)$ and finite (and known) variance $\sigma_i^2$. Also, we let $B$ be the total number of batches, each of size $K$, such that $T = BK$. Furthermore, we denote the last time step in each batch $b \in [B]$ as $t_b$, i.e., $t_b = bK$. We also define the \textit{history} $H_t$ as the sequence of actions and rewards until time step $t$, such that $H_t = \left(\bm{a}_1, r_1 (\bm{a}_1), \dots, \bm{a}_{t - 1}, r_{t - 1}(\bm{a}_{t - 1})\right)$. Since the the actions and rewards are random variables, $H_t$ is a random variable as well. We denote a realization of $H_t$ as $H$, i.e., a \textit{fixed} history of actions and rewards. 

\begin{center}
\begin{algorithm}[H]
\caption{Batched Thompson Sampling for Combinatorial Semi-Bandits}
\label{alg:b_ts_cmab}
\begin{algorithmic}[1]
\Require Time horizon $T$, number of batches $B$, prior parameters $\bm{\mu}_0, \bm{\varsigma}_0$.
\For{$b \leftarrow 1, \dots, B$}
    \For{$t \leftarrow t_{b-1}+1, \dots, t_b$}
        \For{$i \in \mathcal{A}$}
            \State $\Tilde{\theta}_i \leftarrow$ Sample from posterior $\mathcal{N}\left(\mu_{i, t_{b-1}}, \varsigma^2_{i,t_{b-1}}\right)$
        \EndFor
        \State $\bm{a}_t \leftarrow \arg \max_{\bm{a} \in \mathcal{I}} f_{\Tilde{\bm{\theta}}} (\bm{a})$
        \State Play super-arm $\bm{a}_t$
    \EndFor
    \State Observe batched rewards $r_{t_{b-1}+1}, \dots r_{t_b}$. Append corresponding arms and rewards to the history of selected super-arms and received rewards, such that $H_{t_b + 1} = \left(\bm{a}_1, r_1 (\bm{a}_1), \dots, \bm{a}_{t_b}, r_{t_b}(\bm{a}_{t_b})\right)$.
    \State Compute posterior parameters $\bm{\mu}_{t_{b}}, \bm{\varsigma}_{t_{b}}$ given the history $H_{t_b + 1}$.
    
\EndFor
\end{algorithmic}
\end{algorithm}
\end{center}

In this problem setting and algorithm, the rewards for all arms performed during a batch are received at the end of that batch. Hence, in each time step, parameters are sampled from the posterior distribution given the rewards observed at the end of the \emph{previous batch}. 

\subsubsection{Regret analysis}

We analyze the regret of this algorithm in the proof of Theorem \ref{thm:b_ts_cmab_regret}, where 

\begin{theorem}
\label{thm:b_ts_cmab_regret}
For Algorithm \ref{alg:b_ts_cmab}, with horizon $T$ and batch size $K$, we have $\text{BayesRegret}(T) = \Tilde{\mathcal{O}}(\cardA \; K + \cardA \sqrt{T})$.
\end{theorem}

In order to prove Theorem \ref{thm:b_ts_cmab_regret}, we need a few intermediary lemmas and assumptions. For base arm $i$, let $\hat{\theta}_{i,t}$ be the average reward of $i$ until time step $t$, and $N_t (i)$ be the number plays of $i$ until time step $t$.

\begin{assumption}\label{asm:finite_variance}
For each base arm $i \in \mathcal{A}$, the variance $\sigma_i^2$ is finite, and $\sigma_i^2 \leq 1$.
\end{assumption}

Since we assume that the variance $\sigma_i^2$ of each base arm $i \in \mathcal{A}$ is finite, we let, for convenience of notation, $\sigma_i^2 \leq 1$ for all $i \in \mathcal{A}$ (which can be achieved by scaling the feedback distributions of all base arms). 

\begin{assumption}\label{asm:low_dimensional}
Given the horizon $T$ and the number of base arms $\vert \mathcal{A} \vert$, we have $T \geq \cardA$.
\end{assumption}

\begin{assumption}\label{asm:initial_plays}
Each base arm $i \in \cardA$ has been played once initially, such that $N_0 (i) = 1$.
\end{assumption}

Assumptions \ref{asm:low_dimensional} and \ref{asm:initial_plays} are mainly for convenience, to reduce the complexity of the proofs, whereas the finite variance assumption is needed for the concentration inequality we utilize in the proof of Lemma \ref{lem:b_ts_cmab_bad_probability_bound}. We begin the analysis by defining upper and lower confidence bounds (for a super-arm $\bm{a}$ and history $H_t$, as defined in Algorithm \ref{alg:b_ts_cmab}):
\begin{align*}
    &U(\bm{a}, H_t) := f_{\hat{\bm{\theta}}_{t-1}} (\bm{a}) + \sum_{i \in \bm{a}} \sqrt{\frac{8 \log T}{N_{t-1} (i)}} \\ 
    &L(\bm{a}, H_t) := f_{\hat{\bm{\theta}}_{t-1}} (\bm{a}) - \sum_{i \in \bm{a}} \sqrt{\frac{8 \log T}{N_{t-1} (i)}} .
\end{align*}

Using these definitions, we can decompose the regret in a way similar to \cite{russo2014learning} as follows:

\begin{lemma}
\label{lem:b_ts_cmab_regret_decomposition}
Algorithm \ref{alg:b_ts_cmab} has %
\begin{align*}
    \text{BayesRegret}(T) = &\sum_{b=1}^B \sum_{t=t_{b-1} +1}^{t_b} \E{U(\bm{a}_t, H_{t_{b-1} +1}) - L(\bm{a}_t, H_{t_{b-1} +1})} + \\ 
    &\sum_{b=1}^B \sum_{t=t_{b-1} +1}^{t_b} \E{L(\bm{a}_t, H_{t_{b-1} +1}) - f_{\bm{\theta}^*} (\bm{a}_t)} + \\  
    &\sum_{b=1}^B \sum_{t=t_{b-1} +1}^{t_b} \E{f_{\bm{\theta}^*} (\bm{a}^*) - U(\bm{a}^*, H_{t_{b-1} +1})} .
\end{align*}
\end{lemma}

\begin{proof}

By the definition of Bayesian regret, we have that:
\begin{align*}
    &\text{BayesRegret}(T) \\ 
    &= \sum_{t=1}^T \E{f_{\bm{\theta}^*} (\bm{a}^*) - f_{\bm{\theta}^*} (\bm{a}_t)} \\
    &= \sum_{b=1}^B \sum_{t=t_{b-1} +1}^{t_b} \E{f_{\bm{\theta}^*} (\bm{a}^*) - f_{\bm{\theta}^*} (\bm{a}_t)} \\
    \intertext{(Tower rule)}
    &= \sum_{b=1}^B \sum_{t=t_{b-1} +1}^{t_b} \mathbb{E}_{H \sim P\left(H_{t_{b-1} + 1}\right)}\left[\E{f_{\bm{\theta}^*} (\bm{a}^*) - f_{\bm{\theta}^*} (\bm{a}_t)\cond H_{t_{b-1} + 1} = H}\right] \\
    &= \sum_{b=1}^B \sum_{t=t_{b-1} +1}^{t_b} \mathbb{E}_{H \sim P\left(H_{t_{b-1} + 1}\right)}\left[\E{U(\bm{a}_t, H) - U(\bm{a}_t, H) + f_{\bm{\theta}^*} (\bm{a}^*) - f_{\bm{\theta}^*} (\bm{a}_t)\cond H_{t_{b-1} + 1} = H}\right] \\
    \intertext{(Conditioned on the history $H_{t_{b-1} + 1}$, up to and including the last batch $b-1$, all super-arms $\bm{a}_t$ for $t = t_{b-1}+1, \dots, t_b$ and the optimal super-arm $\bm{a}^*$ are identically distributed. We have that $\E{U(\bm{a}_t, H) \,\vert\, H_{t_{b-1} +1} = H} = \E{U(\bm{a}^*, H) \,\vert\, H_{t_{b-1} +1} = H}$, since $U$ is a deterministic function of a super-arm and a history)}
    &= \sum_{b=1}^B \sum_{t=t_{b-1} +1}^{t_b} \mathbb{E}_{H \sim P\left(H_{t_{b-1} + 1}\right)}\left[\E{U(\bm{a}_t, H) - U(\bm{a}^*, H) + f_{\bm{\theta}^*} (\bm{a}^*) - f_{\bm{\theta}^*} (\bm{a}_t)\cond H_{t_{b-1} + 1} = H}\right] \\
    &= \sum_{b=1}^B \sum_{t=t_{b-1} +1}^{t_b} \E{U(\bm{a}_t, H_{t_{b-1} +1}) - L(\bm{a}_t, H_{t_{b-1} +1})} \\
    &+ \sum_{b=1}^B \sum_{t=t_{b-1} +1}^{t_b} \E{L(\bm{a}_t, H_{t_{b-1} +1}) - f_{\bm{\theta}^*} (\bm{a}_t)} \\
    &+ \sum_{b=1}^B \sum_{t=t_{b-1} +1}^{t_b} \E{f_{\bm{\theta}^*} (\bm{a}^*) - U(\bm{a}^*, H_{t_{b-1} +1})} .
\end{align*}
\end{proof}

To bound the last two terms of the decomposed Bayesian regret, we use the following lemma.

\begin{lemma}
\label{lem:b_ts_cmab_bad_event_terms}
For any batch $b = 1, \dots, B$ and any time step $t = t_{b-1} + 1, \dots, t_b$, we have that 

\begin{align*}
    &\E{L(\bm{a}_t, H_{t_{b-1} +1}) - f_{\bm{\theta}^*} (\bm{a}_t)} \leq \frac{2}{T} \\
    &\E{f_{\bm{\theta}^*} (\bm{a}^*) - U(\bm{a}^*, H_{t_{b-1} +1})} \leq \frac{2}{T} .
\end{align*}
\end{lemma}

\begin{proof}
Both $\E{L(\bm{a}_t, H_{t_{b-1} +1}) - f_{\bm{\theta}^*} (\bm{a}_t)} \leq \frac{2}{T}$ and $\E{f_{\bm{\theta}^*} (\bm{a}^*) - U(\bm{a}^*, H_{t_{b-1} +1})} \leq \frac{2}{T}$ are proven in the same way, so we focus on the first inequality:
\begin{align*}
    &\E{L(\bm{a}_t, H_{t_{b-1} +1}) - f_{\bm{\theta}^*} (\bm{a}_t)} \\
    &= \E{f_{\hat{\bm{\theta}}_{t_{b-1}}} (\bm{a}_t) - f_{\bm{\theta}^*} (\bm{a}_t) - \sum_{i \in \bm{a}_t} \sqrt{\frac{8 \log T}{N_{t_{b-1}} (i)}}} \\
    &= \E{\sum_{i \in \bm{a}_t} \hat{\theta}_{i, t_{b-1}} - \sum_{i \in \bm{a}_t} \theta^*_i - \sum_{i \in \bm{a}_t} \sqrt{\frac{8 \log T}{N_{t_{b-1}} (i)}}} \\
    & = \E{\sum_{i \in \bm{a}_t} \left( \hat{\theta}_{i, t_{b-1}} - \theta^*_i -  \sqrt{\frac{8 \log T}{N_{t_{b-1}} (i)}}\right)}\\
    \intertext{(We let $[x]^+ := \max(0,x)$)}
    &\leq \E{\sum_{i \in \bm{a}_t} \left[ \hat{\theta}_{i, t_{b-1}} - \theta^*_i -  \sqrt{\frac{8 \log T}{N_{t_{b-1}} (i)}}\right]^+} \\
    &\leq \E{\sum_{i \in \mathcal{A}} \left[ \hat{\theta}_{i, t_{b-1}} - \theta^*_i -  \sqrt{\frac{8 \log T}{N_{t_{b-1}} (i)}}\right]^+} \\
    &= \sum_{i \in \mathcal{A}} \E{\left[ \hat{\theta}_{i, t_{b-1}} - \theta^*_i -  \sqrt{\frac{8 \log T}{N_{t_{b-1}} (i)}}\right]^+} \\
    &\leq \sum_{i \in \mathcal{A}} \E{\left[ \vert \hat{\theta}_{i, t_{b-1}} - \theta^*_i \vert -  \sqrt{\frac{8 \log T}{N_{t_{b-1}} (i)}}\right]^+} \\
    &= \sum_{i \in \mathcal{A}} \E{ \vert \hat{\theta}_{i, t_{b-1}} - \theta^*_i \vert -  \sqrt{\frac{8 \log T}{N_{t_{b-1}} (i)}} \cond \vert \hat{\theta}_{i, t_{b-1}} - \theta^*_i \vert \geq  \sqrt{\frac{8 \log T}{N_{t_{b-1}} (i)}} } \cdot \prob{\vert \hat{\theta}_{i, t_{b-1}} - \theta^*_i \vert \geq  \sqrt{\frac{8 \log T}{N_{t_{b-1}} (i)}}} \\
    \intertext{(By Lemma \ref{lem:b_ts_cmab_bad_expectation_bound} and Lemma \ref{lem:b_ts_cmab_bad_probability_bound})}
    & \leq \sum_{i \in \mathcal{A}} \frac{2}{T^2} \\
    \intertext{(By Assumption \ref{asm:low_dimensional}, $\cardA \leq T$)}
    &\leq \frac{2}{T} .
\end{align*}
\end{proof}

To bound the expected overestimation (or, correspondingly, underestimation) in the second-to-last inequality of the proof for Lemma \ref{lem:b_ts_cmab_bad_event_terms}, we derive two intermediate results in Lemma \ref{lem:b_ts_cmab_bad_expectation_bound} and Lemma \ref{lem:b_ts_cmab_bad_probability_bound}. For both of the lemmas, we let $\Bar{\nu}_{i,x}$ be the average reward of base arm $i$ over the first $x$ times it has been played, i.e., contained in any played super-arm. In other words, for each batch $b \in [B]$ we have that $\hat{\theta}_{i,t_{b-1}} = \Bar{\nu}_{i,N_{t_{b-1}} (i)}$. Additionally, for the proofs of both lemmas, we note that the average $\Bar{\nu}_{i,x}$ is Gaussian with mean $\theta^*_i$ and variance $\sigma_i^2 / x$. Since, by Assumption \ref{asm:finite_variance}, we have that $\sigma_i^2 \leq 1$, this implies that $(\Bar{\nu}_{i,x} - \theta^*_i)$ has mean 0 and variance $\leq 1$.

\begin{lemma}
\label{lem:b_ts_cmab_bad_expectation_bound}
For any batch $b \in [B]$ and base arm $i \in \mathcal{A}$, it holds that
\begin{equation*}
    \E{\vert\hat{\theta}_{i, t_{b-1}} - \theta^*_i\vert -  \sqrt{\frac{8 \log T}{N_{t_{b-1}} (i)}} \cond  \vert \hat{\theta}_{i, t_{b-1}} - \theta^*_i \vert \geq  \sqrt{\frac{8 \log T}{N_{t_{b-1}} (i)}}} \leq 1 .
\end{equation*}
\end{lemma}
\begin{proof}
We have that:

\begin{align}
    &\E{\vert\hat{\theta}_{i, t_{b-1}} - \theta^*_i\vert -  \sqrt{\frac{8 \log T}{N_{t_{b-1}} (i)}} \cond  \vert \hat{\theta}_{i, t_{b-1}} - \theta^*_i \vert \geq  \sqrt{\frac{8 \log T}{N_{t_{b-1}} (i)}}} \nonumber\\
    \intertext{(Tower rule)}
    &= \mathbb{E}_{x \sim P\left({N_{t_{b-1}} (i)}\right)} \left[ \E{\vert \Bar{\nu}_{i,x} - \theta^*_i \vert -  \sqrt{\frac{8 \log T}{x}} \cond \vert \Bar{\nu}_{i,x} - \theta^*_i \vert \geq  \sqrt{\frac{8 \log T}{x}} \;\wedge\; N_{t_{b-1}} (i) = x}  \right] \nonumber\\
    \intertext{($\E{\Bar{\nu}_{i,x}} = \theta^*_i$, hence $\Bar{\nu}_{i,x} - \theta^*_i$ is $0$-mean Gaussian, and $\Bar{\nu}_{i,x} - \theta^*_i \overset{d}{=} \theta^*_i - \Bar{\nu}_{i,x}$)}
    &= \mathbb{E}_{x \sim P\left({N_{t_{b-1}} (i)}\right)} \left[ \E{ \Bar{\nu}_{i,x} - \theta^*_i  -  \sqrt{\frac{8 \log T}{x}} \cond  \Bar{\nu}_{i,x} - \theta^*_i  \geq  \sqrt{\frac{8 \log T}{x}} \;\wedge\; N_{t_{b-1}} (i) = x}  \right] . \label{eq:truncated_expectation}
\end{align}

For any fixed integer $x > 0$, we have that $\left(\Bar{\nu}_{i,x} - \theta^*_i - \sqrt{\frac{8 \log T}{x}}\right)$ is Gaussian with expected value $\left(- \sqrt{\frac{8 \log T}{x}}\right) < 0$. The inner expectation in Eq. \ref{eq:truncated_expectation} is the expected value of the corresponding truncated (below 0) Gaussian distribution, which (by, e.g., Theorem 2 of \cite{horrace2015moments}) is increasing in $\left(- \sqrt{\frac{8 \log T}{x}}\right)$. Consequently,

\begin{align}
    &\E{ \Bar{\nu}_{i,x} - \theta^*_i  -  \sqrt{\frac{8 \log T}{x}} \cond  \Bar{\nu}_{i,x} - \theta^*_i - \sqrt{\frac{8 \log T}{x}} \geq  0}  \nonumber\\
    &\leq \E{ \Bar{\nu}_{i,x} - \theta^*_i \cond  \Bar{\nu}_{i,x} - \theta^*_i \geq  0} \nonumber\\
    \intertext{(Mean of truncated Gaussian, see \cite{horrace2015moments}, with $\E{\Bar{\nu}_{i,x} - \theta^*_i} = 0$ and $\mathbf{Var}\left[\Bar{\nu}_{i,x} - \theta^*_i\right] = \sigma_i^2 / x$)}
    &= \frac{\sigma_i}{\sqrt{x}} \frac{\phi\left(0\right)}{1 - \Phi\left(0\right)} \nonumber\\
    \intertext{(By Assumption \ref{asm:finite_variance})}
    &\leq \frac{\phi(0)}{1 - \Phi(0)} \; \approx \; 0.798 \nonumber\\
    &\leq 1 . \label{eq:truncated_bound}
\end{align}

The claim follows by bounding the inner expectation of Eq. \ref{eq:truncated_expectation} using Eq. \ref{eq:truncated_bound}.
\end{proof}

\begin{lemma}
\label{lem:b_ts_cmab_bad_probability_bound}
$\prob{\exists b \in [B] \; \exists i \in \mathcal{A}, \vert \theta^*_i - \hat{\theta}_{i,t_{b-1}} \vert \geq \sqrt{\frac{8 \log T}{N_{t_{b-1}} (i)}}} \leq \frac{2}{T^2}$
\end{lemma}

\begin{proof}

We perform a standard concentration analysis using union bounds and Hoeffding inequality, adapted for the batched feedback setting:

\begin{align}
    &\prob{\exists b \in [B] \; \exists i \in \mathcal{A}, \vert \theta^*_i - \hat{\theta}_{i,t_{b-1}} \vert \geq \sqrt{\frac{8 \log T}{N_{t_{b-1}} (i)}}} \nonumber\\
    &\leq \prob{\exists x \in [t_{B-1}] \; \exists i \in \mathcal{A}, \vert \theta^*_i - \Bar{\nu}_{i,x} \vert \geq \sqrt{\frac{8 \log T}{x}}} \nonumber\\
    &\leq \prob{\exists x \in [T] \; \exists i \in \mathcal{A}, \vert \theta^*_i - \Bar{\nu}_{i,x} \vert \geq \sqrt{\frac{8 \log T}{x}}} \nonumber\\
    \intertext{(Union bound)}
    &\leq \sum_{x = 1}^T \sum_{i \in \mathcal{A}} \prob{\vert \theta^*_i - \Bar{\nu}_{i,x} \vert \geq \sqrt{\frac{8 \log T}{x}}} \nonumber\\
    \intertext{(Hoeffding inequality for 1-subgaussian random variables, since $\theta^*_i - \Bar{\nu}_{i,x}$ is Gaussian with mean 0 and variance $\leq 1$, by Assumption \ref{asm:finite_variance})}
    &\leq \sum_{x = 1}^T \sum_{i \in \mathcal{A}} \frac{2}{T^4} \nonumber\\
    \intertext{(By Assumption \ref{asm:low_dimensional}, $\cardA \leq T$)}
    &\leq \frac{2}{T^2} . \nonumber
\end{align}

\end{proof}

With the last two terms of the regret decomposition of Lemma \ref{lem:b_ts_cmab_regret_decomposition} bounded using Lemma \ref{lem:b_ts_cmab_bad_event_terms}, we may focus on the first term. We can bound it in the following way:
\begin{lemma}
\label{lem:b_ts_cmab_good_event_term}
$\sum_{b=1}^B \sum_{t=t_{b-1} +1}^{t_b} \E{U(\bm{a}_t, H_{t_{b-1} +1}) - L(\bm{a}_t, H_{t_{b-1} +1})} \leq 4\sqrt{8 \log T} \cdot ( \cardA \; K  + \cardA \sqrt{T}) .$
\end{lemma}

\begin{proof}
\begin{align*}
    &\sum_{b=1}^B \sum_{t=t_{b-1} +1}^{t_b} \E{U(\bm{a}_t, H_{t_{b-1} +1}) - L(\bm{a}_t, H_{t_{b-1} +1})} \\
    &= 2 \sum_{b=1}^B \sum_{t=t_{b-1} +1}^{t_b} \E{\sum_{i \in \bm{a}_t} \sqrt{\frac{8 \log T}{N_{t_{b-1}} (i)}}} \\
    &= 2 \sum_{b=1}^B \sum_{t=t_{b-1} +1}^{t_b} \left(\E{\sum_{i \in \bm{a}_t} \left(\sqrt{\frac{8 \log T}{N_{t_{b-1}} (i)}} - \sqrt{\frac{8 \log T}{N_{t-1} (i)}}\right)} + \E{\sum_{i \in \bm{a}_t} \sqrt{\frac{8 \log T}{N_{t-1} (i)}}}\right) 
    \\
    &= 2 \sum_{b=1}^B \sum_{t=t_{b-1} +1}^{t_b} \E{\sum_{i \in \bm{a}_t} \left(\sqrt{\frac{8 \log T}{N_{t_{b-1}} (i)}} - \sqrt{\frac{8 \log T}{N_{t-1} (i)}}\right)} + 2 \sum_{t=1}^T \E{\sum_{i \in \bm{a}_t} \sqrt{\frac{8 \log T}{N_{t-1} (i)}}} .
\end{align*}

The first term in the last expression above bounds the regret resulting from the batch delays, while the second term bounds the regret of the Thompson Sampling algorithm for the corresponding non-batched combinatorial semi-bandit setting. We start by bounding the first term:

\begin{align*}
    &2 \sum_{b=1}^B \sum_{t=t_{b-1} +1}^{t_b} \E{\sum_{i \in \bm{a}_t} \left(\sqrt{\frac{8 \log T}{N_{t_{b-1}} (i)}} - \sqrt{\frac{8 \log T}{N_{t-1} (i)}}\right)} \\
    &\leq 2 \sum_{b=1}^B \sum_{t=t_{b-1} +1}^{t_b} \E{\sum_{i \in \bm{a}_t} \left(\sqrt{\frac{8 \log T}{N_{t_{b-1}} (i)}} - \sqrt{\frac{8 \log T}{N_{t_b} (i)}}\right)} \\
    &\leq 2 \sum_{b=1}^B \sum_{t=t_{b-1} +1}^{t_b} \E{\sum_{i \in \mathcal{A}} \left(\sqrt{\frac{8 \log T}{N_{t_{b-1}} (i)}} - \sqrt{\frac{8 \log T}{N_{t_b} (i)}}\right)}\\
    &= 2 \sum_{b=1}^B \sum_{t=t_{b-1} +1}^{t_b} \sum_{i \in \mathcal{A}} \E{ \left(\sqrt{\frac{8 \log T}{N_{t_{b-1}} (i)}} - \sqrt{\frac{8 \log T}{N_{t_b} (i)}}\right)}\\
    &= 2 K \sum_{b=1}^B \sum_{i \in \mathcal{A}} \E{ \left(\sqrt{\frac{8 \log T}{N_{t_{b-1}} (i)}} - \sqrt{\frac{8 \log T}{N_{t_b} (i)}}\right)} \\
    &= 2 K \sum_{i \in \mathcal{A}} \E{ \left(\sqrt{\frac{8 \log T}{N_{t_{0}} (i)}} - \sqrt{\frac{8 \log T}{N_{t_B} (i)}}\right)} \\
    &\leq 2 K \sum_{i \in \mathcal{A}} \E{ \left(\sqrt{\frac{8 \log T}{N_{t_{0}} (i)}} \right)} \\
    \intertext{(By Assumption \ref{asm:initial_plays}, $N_{t_0} = 1$)}
    &= 2 K \sum_{i \in \mathcal{A}} \E{ \left(\sqrt{8 \log T} \right)} \\
    &= 2 K \cardA \sqrt{8 \log T}
\end{align*}

We can then continue by bounding the second term:
\begin{align*}
    &2 \sum_{t=1}^T \E{\sum_{i \in \bm{a}_t} \sqrt{\frac{8 \log T}{N_{t-1} (i)}}}\\
    &= 2\sqrt{8 \log T} \sum_{t\in [T]} \E{\sum_{i\in \bm{a}_t} \frac{1}{\sqrt{N_{t-1}(i)}}}\\
    &= 2\sqrt{8 \log T} \sum_{i\in \mathcal{A}} \E{\sum_{t:i\in \bm{a}_t} \frac{1}{\sqrt{N_{t-1}(i)}}}\\
    \intertext{(See the proof of Lemma 1 in \cite{russo2014learning})}
    &\leq 2\sqrt{8 \log T} \sum_{i\in \mathcal{A}} \E{2 \sqrt{N_{T}(i)}} \\
    \intertext{(Cauchy-Schwarz inequality)}
    &\leq 2\sqrt{8 \log T} \cdot \E{2 \sqrt{\cardA \sum_{i\in \mathcal{A}} N_{T}(i)}} \\
    &\leq 2\sqrt{8 \log T} \cdot \E{2 \sqrt{\cardA^2 T}} \\
    &= 4 \cardA \sqrt{8  T \log T} .
\end{align*}

This completes the proof of the lemma.

\end{proof}

With these lemmas, we can finish the proof of Theorem \ref{thm:b_ts_cmab_regret}:

\begin{proof}[Proof of Theorem \ref{thm:b_ts_cmab_regret}]
We bound terms in the regret decomposition of Lemma \ref{lem:b_ts_cmab_regret_decomposition} using Lemma \ref{lem:b_ts_cmab_bad_event_terms} and Lemma \ref{lem:b_ts_cmab_good_event_term}, such that:
\begin{align*}
    &\text{BayesRegret}(T) \\
    &\leq  4 + 4\sqrt{8 \log T} \cdot ( \cardA \; K  + \cardA \sqrt{T}) \\
    &\leq  \Tilde{\mathcal{O}}(\cardA \; K  + \cardA \sqrt{T}) .
\end{align*}
\end{proof}

The result in Theorem \ref{thm:b_ts_cmab_regret} applies to a setting with unbounded Gaussian rewards. While general, it does not directly correspond to either of the models described in Section \ref{sec:energy_model}. However, it is straightforward to modify the proof so that it applies to a setting with rectified Gaussian base arm rewards (i.e., for a batched version of Algorithm \ref{alg:ts_optimization_objective}).

\begin{proposition}
\label{prop:b_ts_cmab_regret_rect}
The Bayesian regret of Algorithm \ref{alg:b_ts_cmab}, modified to sample arms as in Algorithm \ref{alg:ts_optimization_objective}, with horizon $T$ and batch size $K$, satisfies $\text{BayesRegret}(T) = \Tilde{\mathcal{O}}(\cardA \; K + \cardA \sqrt{T})$.
\end{proposition}

\begin{proof}
Let $f^R_{\bm{\theta}}(\bm{a}) := - \sum_{i \in \bm{a}} \mathbb{E}_{z_i \sim \mathcal{N}^R(-\theta_i, \sigma_i^2)} \left[ z_i \right]$ be the expected super-arm reward function for a combinatorial semi-bandit with rectified Gaussian base arm feedback. Note that, to connect the super-arm reward function to the rectified Gaussian model in Section \ref{sec:rectified_gaussian_model} and the online shortest path problem formulation, we let base arm feedback be negative, with rectification above 0. The first term of the regret decomposition in Lemma \ref{lem:b_ts_cmab_regret_decomposition} is bounded in Lemma \ref{lem:b_ts_cmab_good_event_term} using only the confidence width term of the upper and lower confidence bounds, not involving the estimated expected super-arm rewards. Hence, under the assumption that we can use the same confidence bounds as in the (non-rectified) Gaussian setting, we only need to ensure that the bounds of the last two terms of the regret decomposition still hold. We can do this with a modification of the proof of Lemma \ref{lem:b_ts_cmab_bad_event_terms}.

\begin{align}
    &\E{L(\bm{a}_t, H_{t_{b-1} +1}) - f_{\bm{\theta}^*}^R (\bm{a}_t)} \nonumber\\
    &= \E{f_{\hat{\bm{\theta}}_{t_{b-1}}}^R (\bm{a}_t) - f_{\bm{\theta}^*}^R (\bm{a}_t) - \sum_{i \in \bm{a}_t} \sqrt{\frac{8 \log T}{N_{t_{b-1}} (i)}}} \nonumber\\
    &= \E{-\sum_{i \in \bm{a}_t} \mathbb{E}_{\hat{z}_i \sim \mathcal{N}^R(-\hat{\theta}_{i, t_{b-1}}, \sigma_i^2)} \left[ \hat{z}_i \right]  + \sum_{i \in \bm{a}_t} \mathbb{E}_{z_i \sim \mathcal{N}^R(-\theta^*, \sigma_i^2)} \left[ z_i \right] - \sum_{i \in \bm{a}_t} \sqrt{\frac{8 \log T}{N_{t_{b-1}} (i)}}} \nonumber\\
    &= \E{\sum_{i \in \bm{a}_t} \left( -\mathbb{E}_{\hat{z}_i \sim \mathcal{N}^R(-\hat{\theta}_{i, t_{b-1}}, \sigma_i^2)} \left[ \hat{z}_i \right] + \mathbb{E}_{z_i \sim \mathcal{N}^R(-\theta^*, \sigma_i^2)} \left[ z_i \right] -  \sqrt{\frac{8 \log T}{N_{t_{b-1}} (i)}}\right)} \nonumber\\
    \intertext{(We have that $- \mathbb{E}_{\hat{z}_i \sim \mathcal{N}^R(-\hat{\theta}_{i, t_{b-1}}, \sigma_i^2)} \left[ \hat{z}_i \right] + \mathbb{E}_{z_i \sim \mathcal{N}^R(-\theta^*, \sigma_i^2)} \left[ z_i \right] \leq \hat{\theta}_{i, t_{b-1}} - \theta^*_i$, since $0 < \frac{\partial}{\partial \theta} \mathbb{E}_{z \sim \mathcal{N}^R(\theta, \sigma_i^2)} \left[ z \right] < 1$)}  
    &\leq \E{\sum_{i \in \bm{a}_t} \left( \hat{\theta}_{i, t_{b-1}} - \theta^*_i -  \sqrt{\frac{8 \log T}{N_{t_{b-1}} (i)}}\right)} . \label{eq:rectified_gaussian_mean_difference_bound}
\end{align}

After Eq. \ref{eq:rectified_gaussian_mean_difference_bound}, the rest of the proof of Lemma \ref{lem:b_ts_cmab_bad_event_terms} holds unmodified. Hence, the bound of Theorem \ref{thm:b_ts_cmab_regret} also holds in the case of rectified Gaussian base arm feedback.

\end{proof}

We can extend this result to the multi-agent online shortest path setting through the following corollary (where the set of edges $\mathcal{E}$ corresponds to the set of base arms $\mathcal{A}$ used throughout the proof of Theorem \ref{thm:b_ts_cmab_regret}). We note that recently, a similar result has been derived in \cite{chan2021parallelizing} for frequentist regret in a linear contextual bandit setting. 

\begin{corollary}
\label{cor:ma_regret_bound_ts2}
Let $K$ be the number of agents, $T$ be the horizon and $\text{Regret}_k (T)$ be the regret of each agent $k \in [K]$. In the synchronous multi-agent online shortest path setting (i.e., a fleet of $K$ agents / vehicles working in parallel in each time step), the total fleet regret incurred by invoking Algorithm \ref{alg:b_ts_cmab} satisfies
$\sum_{k=1}^K \text{BayesRegret}_k (T) \leq \Tilde{\mathcal{O}}\left(\cardE \; K + \vert\mathcal{E}\vert \sqrt{TK} \right)$.
\end{corollary}

\begin{proof}
We prove this in the same way as the proof for Theorem \ref{thm:ma_regret_bound_ts}, but use Theorem \ref{thm:b_ts_cmab_regret} instead of the result for QPM-D in \cite{joulani2013online}.
\end{proof}

For completeness, we also formally state the Bayesian regret upper bound mentioned in Section \ref{sec:ts_regret_analysis} as the following corollary of Theorem \ref{thm:b_ts_cmab_regret} and Proposition \ref{prop:b_ts_cmab_regret_rect}, with batch size 1.

\begin{corollary}
\label{cor:graph_regret_bound_ts2}
The Bayesian regret of Algorithm \ref{alg:online_algorithm} is upper bounded by
\begin{align*}
    \text{BayesRegret}(T) \leq \Tilde{\mathcal{O}}\left(\cardE \sqrt{ T}\right) .
\end{align*}
\end{corollary}

This corollary matches the bound from Proposition 3 of \cite{russo2014learning}, which can be applied to any combinatorial semi-bandit problem with a linear super-arm reward function, when seen as a special case of the linear bandit problem. However, our analysis does not assume that the prior distributions have bounded support.

One way to discuss the optimality of the upper bounds derived in Theorem \ref{thm:b_ts_cmab_regret} and Proposition \ref{prop:b_ts_cmab_regret_rect}, is to compare them with existing lower bounds. To our knowledge, there is no established lower bound for the specific setting studied in this work (i.e., the batched feedback combinatorial semi-bandit problem). However, there are related bounds that one could either possibly derive a lower bound from, or discuss the upper bound in terms of. Perchet et al. derived a lower bound (Theorem 4 in \cite{perchet2016batched}) for the excess regret due to the delay in the two-armed bandit problem, which is a special case of our problem. Furthermore, there are lower bounds for the (non-delayed) combinatorial semi-bandit problem (e.g., by Kveton et al., Proposition 2 in \cite{kveton2015}), which induce a mandatory term in any lower bound for this problem. Combining these two will result in a lower bound to which the upper bound we derive in Theorem 3 is not tight in the excess regret term, since the upper bound includes a linear dependence on the number of base arms. We conjecture that it should also be possible to adapt the lower bound (for linear contextual bandits with adversarially generated contexts) by Ren et al. in Theorem 1 of \cite{ren2020batched}, which includes a square-root factor (i.e., $\sqrt{\vert \mathcal{A} \vert}$ with the notation used in our work) for the excess regret term. It is notable that under both of these conjectured lower bounds, the $\Tilde{\mathcal{O}}\left(\sqrt{T}\right)$ term of our upper bound is optimal up to polylogarithmic factors. 

\section{Experimental Results}

In this section, we describe different experimental studies.
For real-world experiments, we extend the simulation framework presented in \cite{tstutorial} to network/graph bandits with general directed graphs, in order to enable exploration scenarios in realistic road networks. 
Furthermore, we add the ability to generate synthetic networks of specified size to this framework, in order to compare with the derived regret bounds (as the ground truth is provided for the synthetic networks). In all experiments, Dijkstra's algorithm is used to compute the shortest paths through the networks.

\subsection{Real-World Experiments} \label{sec:exp:rw}

For the experiments in real-world road networks, we study one scenario with realistic energy consumption distributions handled by the agents using misspecified wide prior distributions, and another scenario where the prior distributions are completely known and utilized by the agents. In the second scenario, the parameters of the underlying energy consumption distributions are sampled from the prior distributions before each each experiment run, whereas in the first scenario, the underlying distributions are fixed over multiple runs. Based on the second setting, we also consider a third setting where the energy consumption of different edges is correlated.

For each of the settings, we perform experiments using data from three cities: Luxembourg, Monaco and Turin. For Luxembourg, specifically, we study two problem instances (denoted \#1 and \#2) with different source and target vertices. We utilize, respectively for each of the cities, the Luxembourg SUMO Traffic (LuST) \cite{codeca2017luxembourg}, Monaco SUMO Traffic (MoST) \cite{MoSTCodeca2017} and Turin SUMO Traffic (TuST) \cite{rapelli2021vehicular} scenarios to provide realistic traffic patterns and vehicle speed distributions for each hour of the day. This is used in conjunction with altitude data \cite{farr2000shuttle}, and vehicle parameters from an electric vehicle.
The resulting graph $\mathcal{G}$ for Luxembourg has $\cardV = 2247$ nodes and $\cardE = 5651$ edges, representing a road network with $955$ km of highways, arterial roads and residential streets. 

We use the default vehicle parameters provided for the energy consumption model in \cite{basso2019energy}, with vehicle front surface area $A = 8$ $\text{m}^2$, air drag coefficient $C_d = 0.7$ and rolling resistance coefficient $C_r = 0.0064$. The vehicle is a medium duty truck with vehicle mass $m = 14750$ kg, which is the curb weight added to half of the payload capacity. 

We approximate the powertrain efficiency during traction by $\eta^{+} = 0.88$ and powertrain efficiency during regeneration by $\eta^{-} = 1.2$. In addition, we use the constant gravitational acceleration $g = 9.81$ $\text{m/s}^2$ and air density $\rho = 1.2$ $\text{kg/m}^3$.

\subsubsection{Prior distribution misspecified by agent} \label{sec:exp_misspecified} %

\begin{figure*}%[!t]
  \centering
  \begin{tabular}{cc}
  \begin{subfigure}[b]{.4\textwidth}
    \centering
    {
      \includegraphics[width=\textwidth]{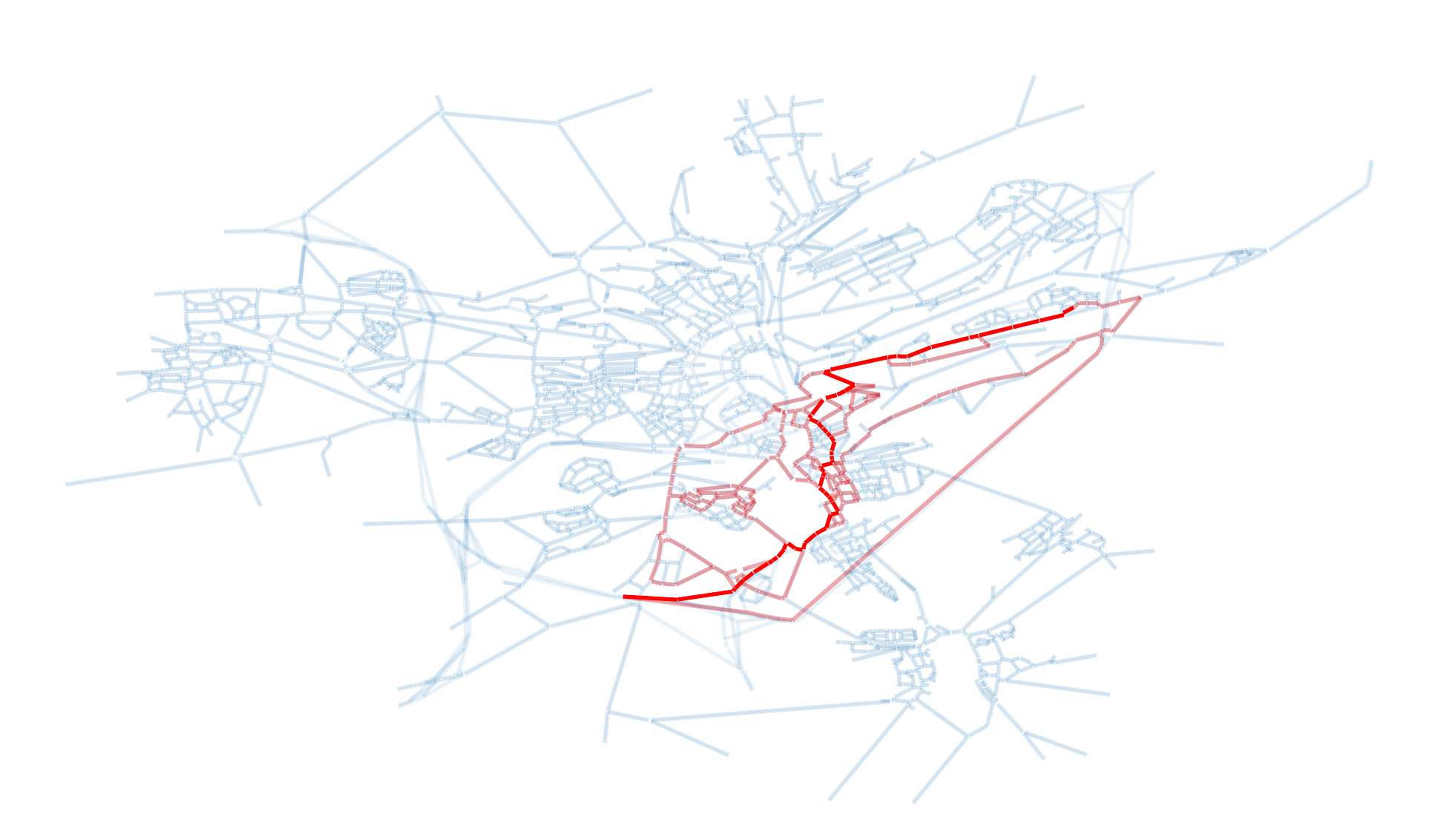}
      \caption{}
      \label{fig:cumulative_graph_luxembourg1_misspecified}
    }
  \end{subfigure}
  &
  \begin{subfigure}[b]{.45\textwidth}
    \centering
    {
      \includegraphics[width=\textwidth]{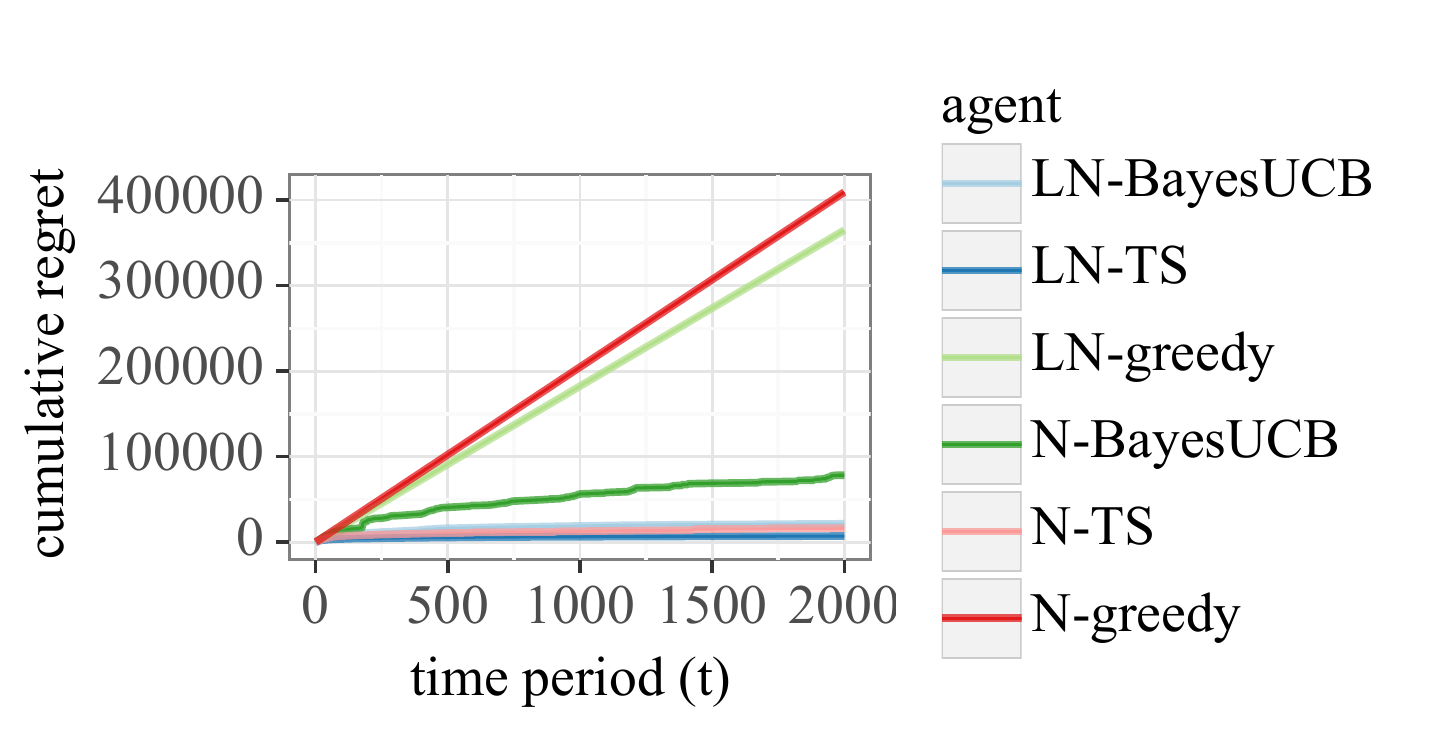}
      \caption{}
      \label{fig:cumulative_regret_luxembourg1_misspecified}
    }
  \end{subfigure}
  \\
  \begin{subfigure}[b]{.4\textwidth}
    \centering
    {
      \includegraphics[width=\textwidth]{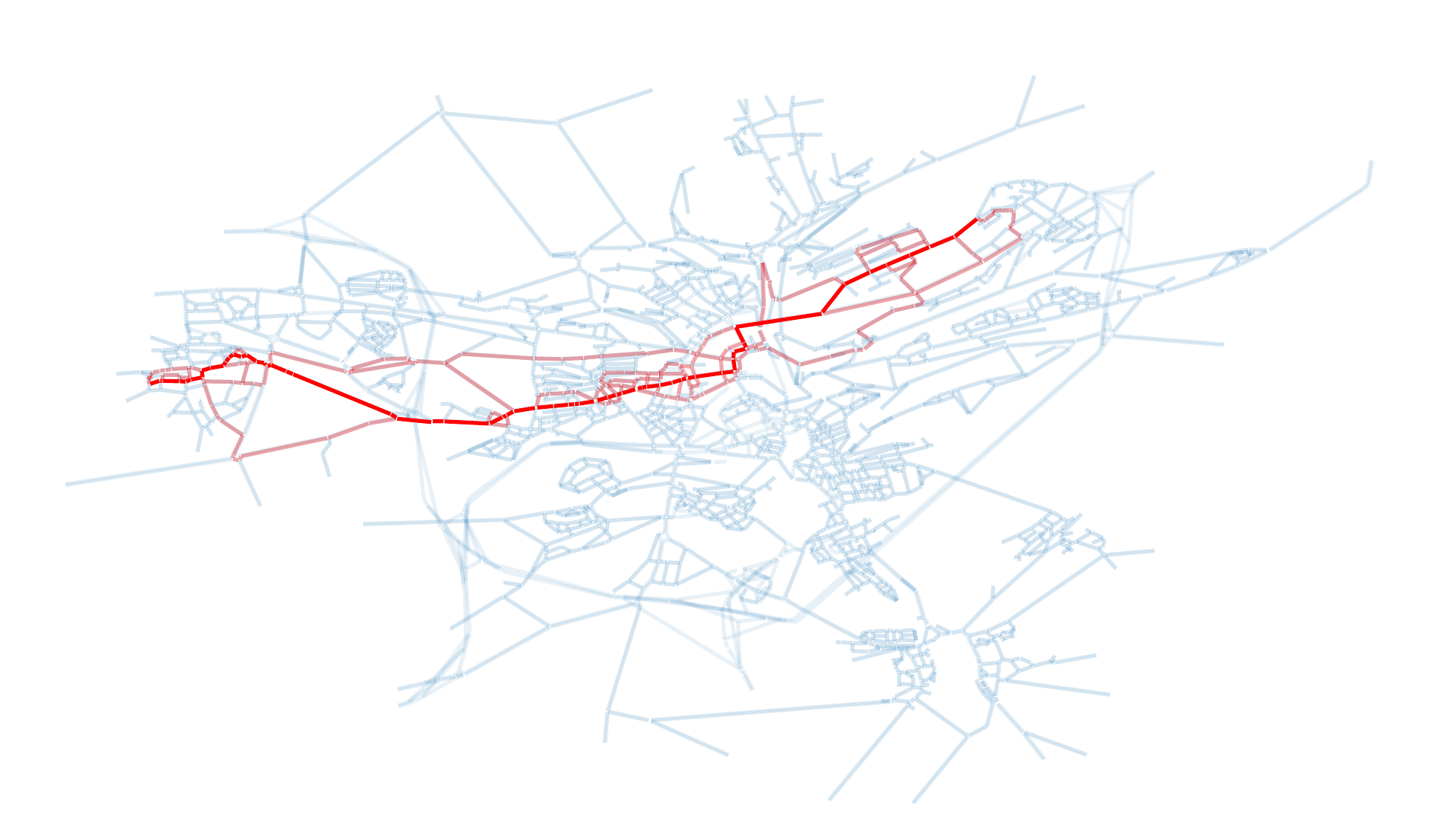}
      \caption{}
      \label{fig:cumulative_graph_luxembourg2_misspecified}
    }
  \end{subfigure}
  &
  \begin{subfigure}[b]{.45\textwidth}
    \centering
    {
      \includegraphics[width=\textwidth]{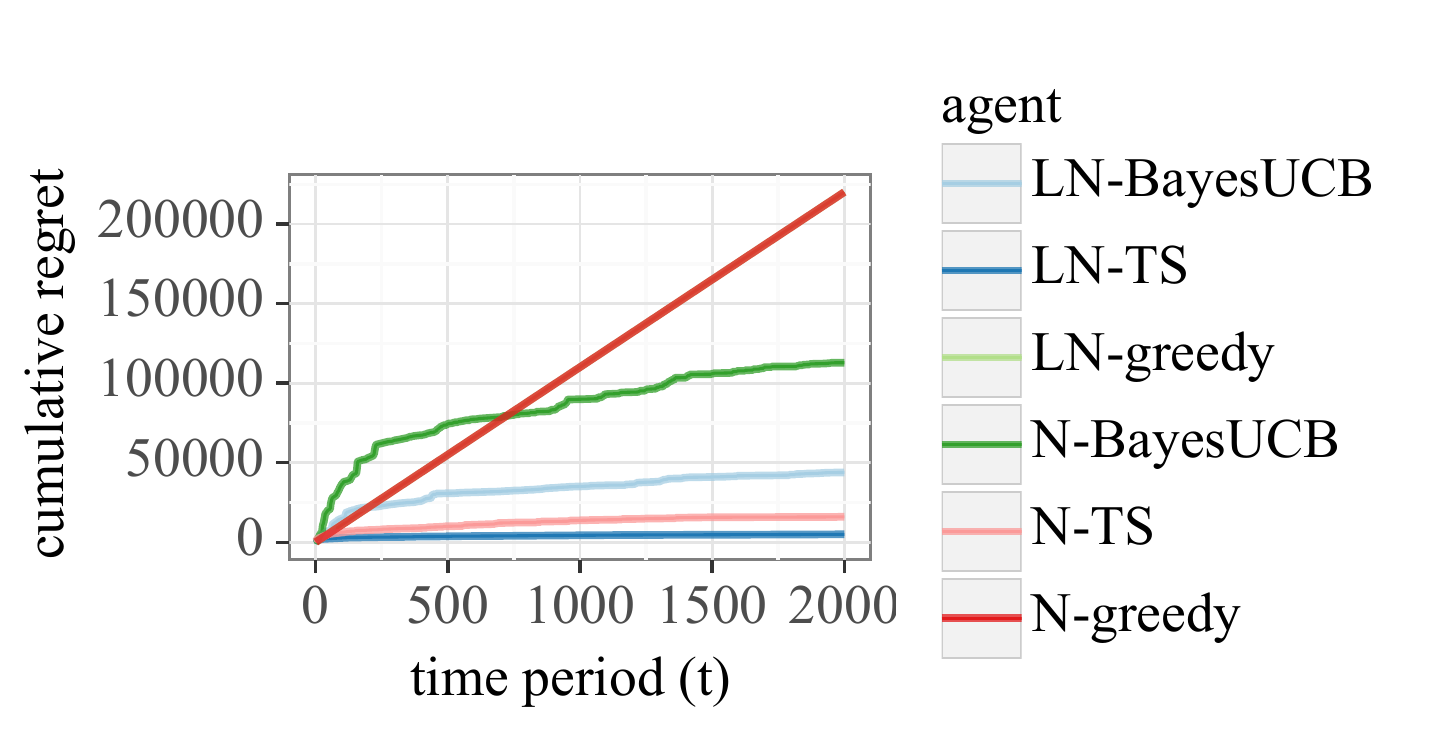}
      \caption{}
      \label{fig:cumulative_regret_luxembourg2_misspecified}
    }
  \end{subfigure}
  \\
  \begin{subfigure}[b]{.4\textwidth}
    \centering
    {
      \includegraphics[width=\textwidth]{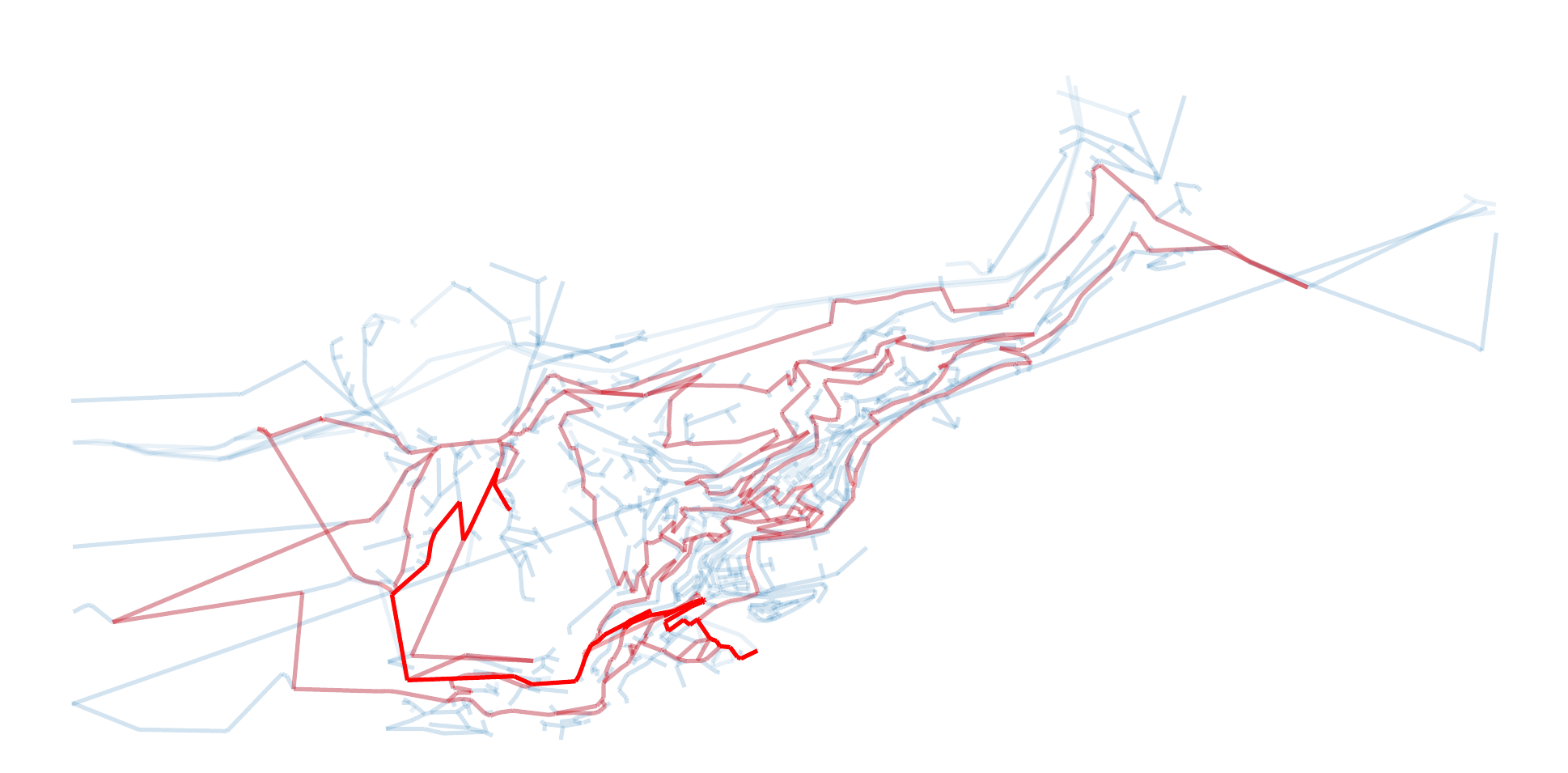}
      \caption{}
      \label{fig:cumulative_graph_monaco1_misspecified}
    }
  \end{subfigure}
  &
  \begin{subfigure}[b]{.45\textwidth}
    \centering
    {
      \includegraphics[width=\textwidth]{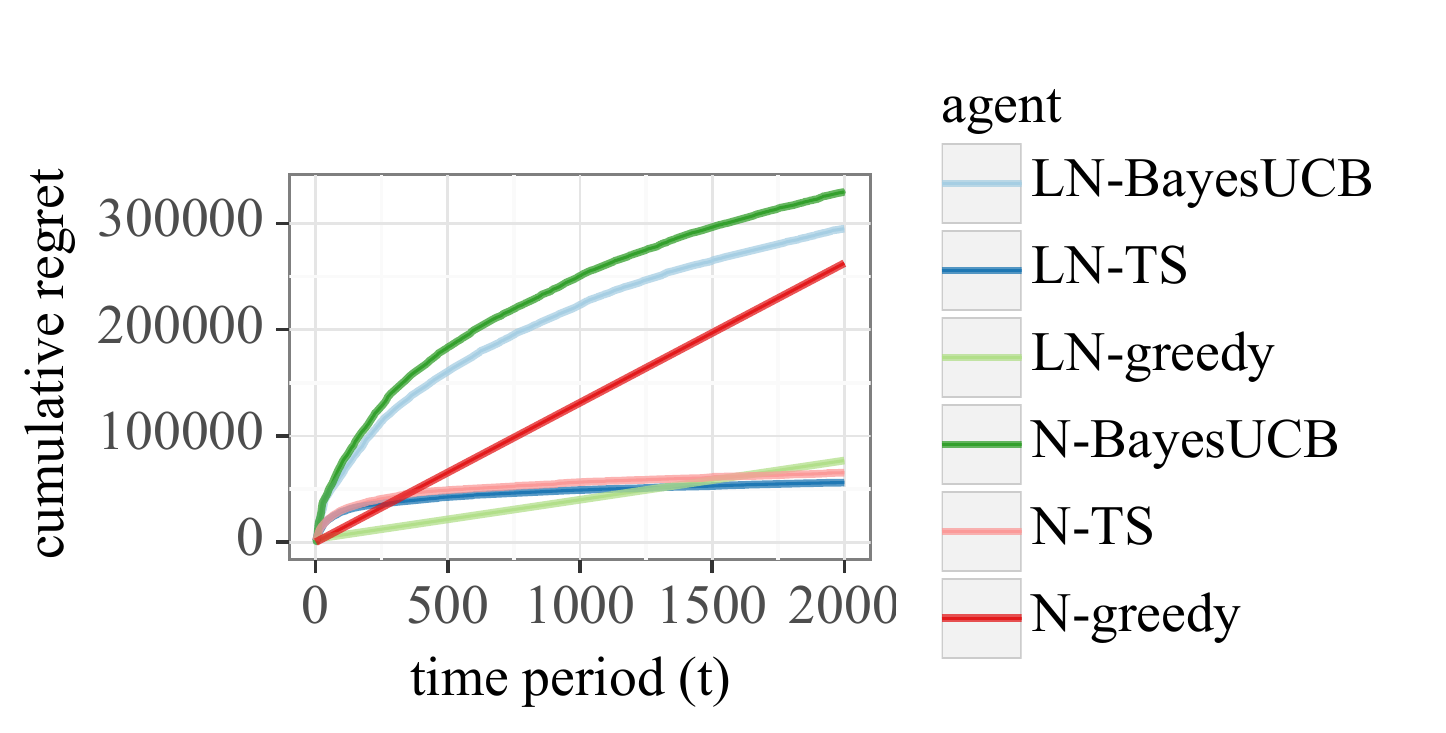}
      \caption{}
      \label{fig:cumulative_regret_monaco1_misspecified}
    }
  \end{subfigure}
  \\
  \begin{subfigure}[b]{.4\textwidth}
    \centering
    {
      \includegraphics[width=\textwidth]{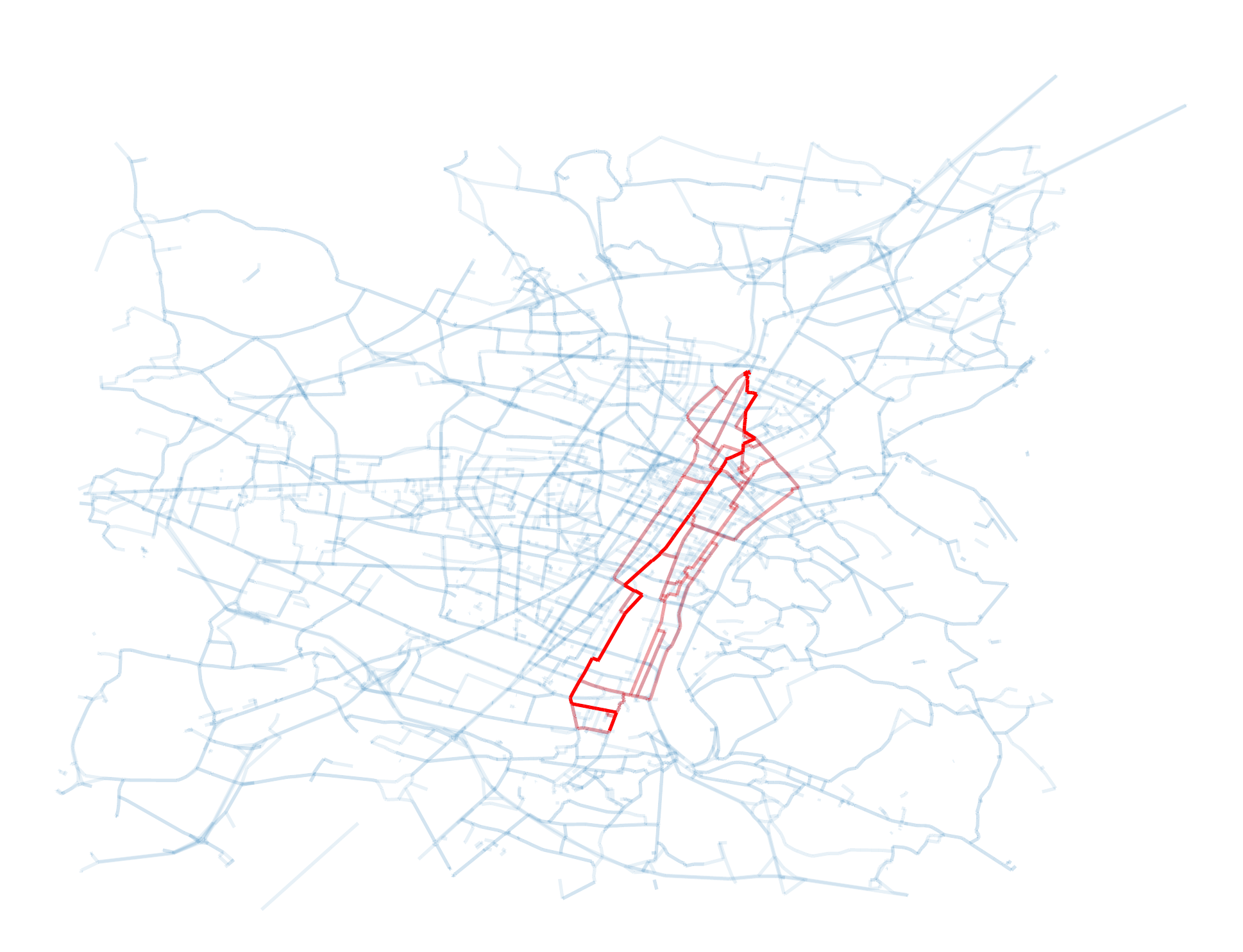}
      \caption{}
      \label{fig:cumulative_graph_turin1_misspecified}
    }
  \end{subfigure}
  &
  \begin{subfigure}[b]{.45\textwidth}
    \centering
    {
      \includegraphics[width=\textwidth]{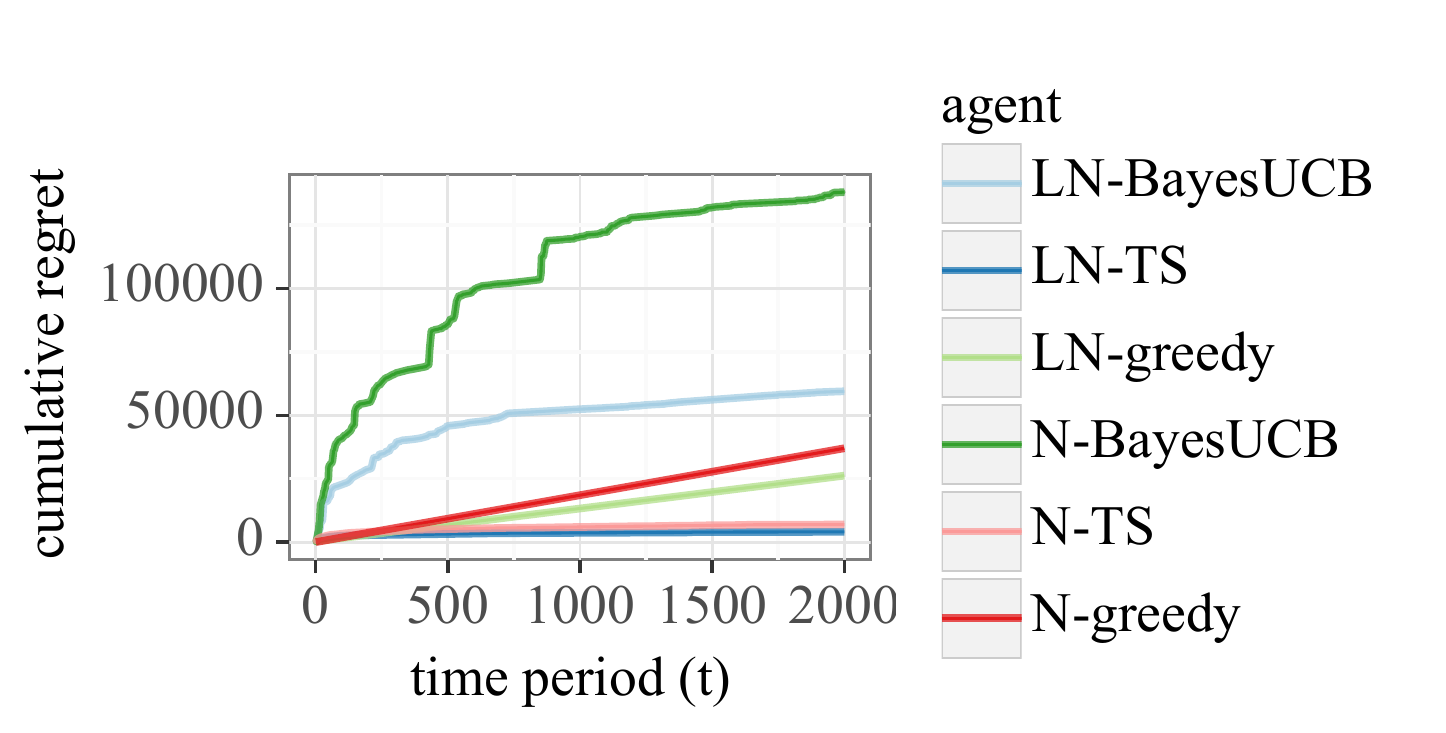}
      \caption{}
      \label{fig:cumulative_regret_turin1_misspecified}
    }
  \end{subfigure}
  \end{tabular}
      \caption{Experimental results on the real-world road networks in the scenario where agents use misspecified priors. For Luxembourg \#1, Luxembourg \#2, Monaco and Turin, respectively, (a), (c), (e) and (g) show the exploration of Thompson Sampling in the road networks, where the red lines indicate the edges visited by the agent during exploration. Paths more frequently traveled are indicated with darker shades of red. Plots (b), (d), (f) and (h) show the average cumulative regret results for Thompson Sampling (TS), BayesUCB and probabilistic greedy algorithms, applied using rectified Gaussian (prefix N) and Log-Gaussian (prefix LN) energy consumption models.}
    \label{fig:exp:results_misspecified}
\end{figure*}

In this set of experiments, with results shown in Figure \ref{fig:exp:results_misspecified} and Table \ref{tbl:standard_misspecified}, we study a scenario where agents do not have access to the true prior distributions of the environment. To simulate the ground truth of the energy consumption, we take the average speed $v_e$ of each edge $e$ from a full 24 hour scenario in each city road network. In particular, for LuST we observe the values during a peak hour (8 AM), with approximately 5500 vehicles active in the network. This hour is selected to increase the risk of traffic congestion, hence finding the optimal path becomes more challenging. We also get the variance of the speed of each road segment from the SUMO scenarios. Using this information, we sample the speed value for each visited edge  and use the energy consumption model to generate the rewards for the arms.

For the probabilistic model, we assume $\sigma_e$ to be proportional to $E_e$ in Eq. \ref{eq:energy_consumption}, such that $\sigma_e^2 = (\varphi E_e)^2$, where we set $\varphi = 0.1$.
For the prior distribution of an edge $e \in \mathcal{E}$, we misspecify it by using the speed limit of $e$ as $v_e$, indicating that the real average speed is unknown. Then $\mu_{e,0} = -E_e$ and $\varsigma_{e,0}^2 = (\vartheta \mu_{e,0})^2$, where $\vartheta = 0.25$.

As a baseline, we consider the greedy algorithm for both the rectified Gaussian and Log-Gaussian models, where the exploration rule is to always choose the path with the lowest currently estimated expected energy consumption, similar to the recent method in \cite{basso2021electric}.

\begin{table}
\centering
\begin{tabular}{||c||c|c||c|c||}%t
\hline
City&\multicolumn{2}{c||}{Luxembourg \#1}&\multicolumn{2}{c||}{Luxembourg \#2}\\
\hline
Agent&AVG&SD&AVG&SD\\
\hline
LN{-}BayesUCB & 21514.64 & 469.83 & 43892.62 & 970.91 \\
LN{-}TS & \textbf{7176.85} & 780.50 & \textbf{4995.93} & 465.94 \\
LN{-}greedy & 364785.95 & 0.00 & 220100.48 & 0.00 \\
N{-}BayesUCB & 78264.31 & 3310.40 & 112856.00 & 1583.15 \\
N{-}TS & 16349.08 & 3082.75 & 16011.48 & 2135.49 \\
N{-}greedy & 409337.40 & 0.00 & 220100.48 & 0.00 \\
\hline
\hline
City&\multicolumn{2}{c||}{Monaco}&\multicolumn{2}{c||}{Turin}\\
\hline
Agent&AVG&SD&AVG&SD\\
\hline
LN{-}BayesUCB & 295057.14 & 2750.97 & 59497.53 & 327.61 \\
LN{-}TS & \textbf{56110.55} & 1822.57 & \textbf{4056.06} & 566.55 \\
LN{-}greedy & 76880.85 & 25947.25 & 26217.88 & 17912.42 \\
N{-}BayesUCB & 329570.78 & 3748.02 & 138046.11 & 2033.91 \\
N{-}TS & 65407.80 & 4739.13 & 6938.36 & 556.36 \\
N{-}greedy & 262622.60 & 0.00 & 37024.61 & 0.00 \\
\hline
\end{tabular}
\vspace{4pt}
\caption{Average and standard deviation of regret at $T = 2000$ of agents with misspecified prior distributions. Bold average values indicate the agent with the lowest regret in each scenario.}\label{tbl:standard_misspecified} %
\end{table}

We run the simulations for the BayesUCB, TS and greedy algorithms with a horizon of $T = 2000$ (i.e., $T=2000$ time steps). Table \ref{tbl:standard_misspecified} and Figures \ref{fig:cumulative_regret_luxembourg1_misspecified}, \ref{fig:cumulative_regret_luxembourg2_misspecified}, \ref{fig:cumulative_regret_monaco1_misspecified} and \ref{fig:cumulative_regret_turin1_misspecified} show the cumulative regret for the rectified Gaussian and Log-Gaussian models (indicated in all tables and figures with prefixes ``N-'' and ``LN-'', respectively, before the name of each algorithm), where the regret is averaged over 10 runs for each agent in each city. The intuition is that the energy saved by using the TS and BayesUCB agents instead of the baseline greedy agent is the difference in regret, expressed in watt-hours. It is clear that Thompson Sampling with the Log-Gaussian model has the best performance in terms of cumulative regret, but the other non-greedy agents also achieve good results. To illustrate that Thompson Sampling explores the road network in a reasonable way, Figures \ref{fig:cumulative_graph_luxembourg1_misspecified}, \ref{fig:cumulative_graph_luxembourg2_misspecified}, \ref{fig:cumulative_graph_monaco1_misspecified} and \ref{fig:cumulative_graph_turin1_misspecified} visualize the road network and the paths visited by this exploration algorithm in each city. Each plot displays all paths visited by the agent during a single experiment, where more frequently traveled paths are indicated with darker shades of red. We observe that in Figures \ref{fig:cumulative_graph_luxembourg1_misspecified}, \ref{fig:cumulative_graph_luxembourg2_misspecified} and \ref{fig:cumulative_graph_turin1_misspecified}, no significant detours are performed, in the sense that most paths are close to the optimal path. While there are some detours shown in Figure \ref{fig:cumulative_graph_monaco1_misspecified}, we note that the distances in Monaco are small compared to the other cities, and that Figure \ref{fig:cumulative_regret_monaco1_misspecified} indicates that the detours do not result in much additional regret. %

\begin{figure*}[!t]
  \centering
  \vspace{2mm}
  \includegraphics[trim={5pt 15pt 0pt 20pt}, width=0.5\textwidth]{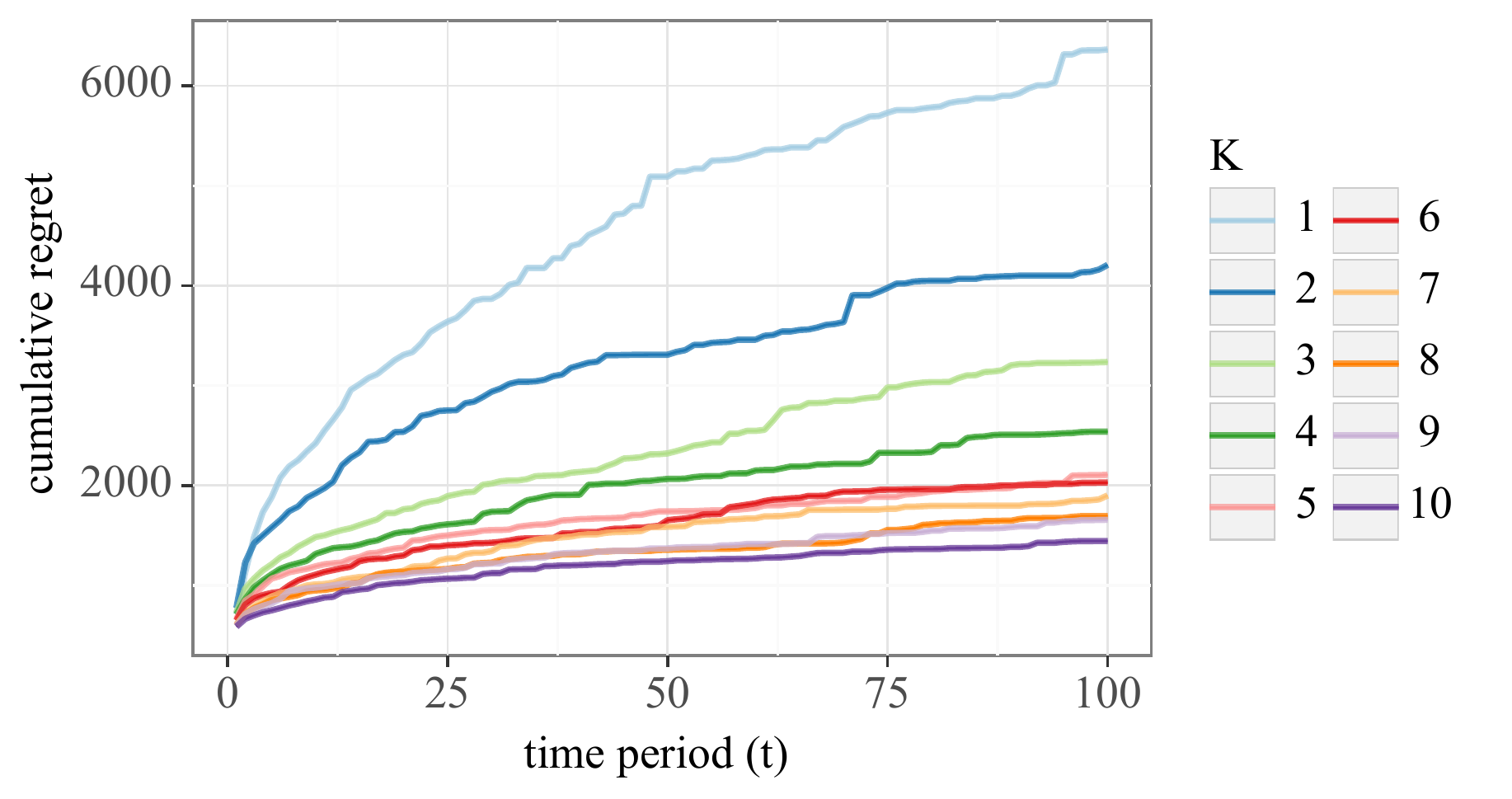}
  \caption{Experimental results in the multi-agent setting, with each line showing the average cumulative regret (horizon $T=100$) for each agent in fleets of size $K$, using Thompson Sampling.}
  \label{fig:multi_cumulative_regret}
\end{figure*}

For the multi-agent case, we use LuST and a horizon of $T = 100$ and 10 scenarios where we vary the number of concurrent agents by $K \in [1, 10]$ in each scenario. The cumulative regret averaged over the agents in these scenarios is shown in Figure \ref{fig:multi_cumulative_regret} for each $K$. In the figure, the final cumulative regret for each agent decreases sharply with the addition of just a few agents to the fleet. This continues until there are five agents, after which there seems to be diminishing returns in adding more agents. While there is some overhead (parallelism cost), just enabling two agents to share knowledge with each other decreases their average cumulative regret at $t=T$ by almost a third. This observation highlights the benefit of providing collaboration early in the exploration process, which is also supported by the regret bound in Corollary \ref{cor:ma_regret_bound_ts2}. %

\subsubsection{Prior distribution known by agent} \label{sec:exp_known_prior}

\begin{figure*}%[!t]
  \centering
  \begin{tabular}{cc}
  \begin{subfigure}[b]{.48\textwidth}
    \centering
    {
      \includegraphics[width=\textwidth]{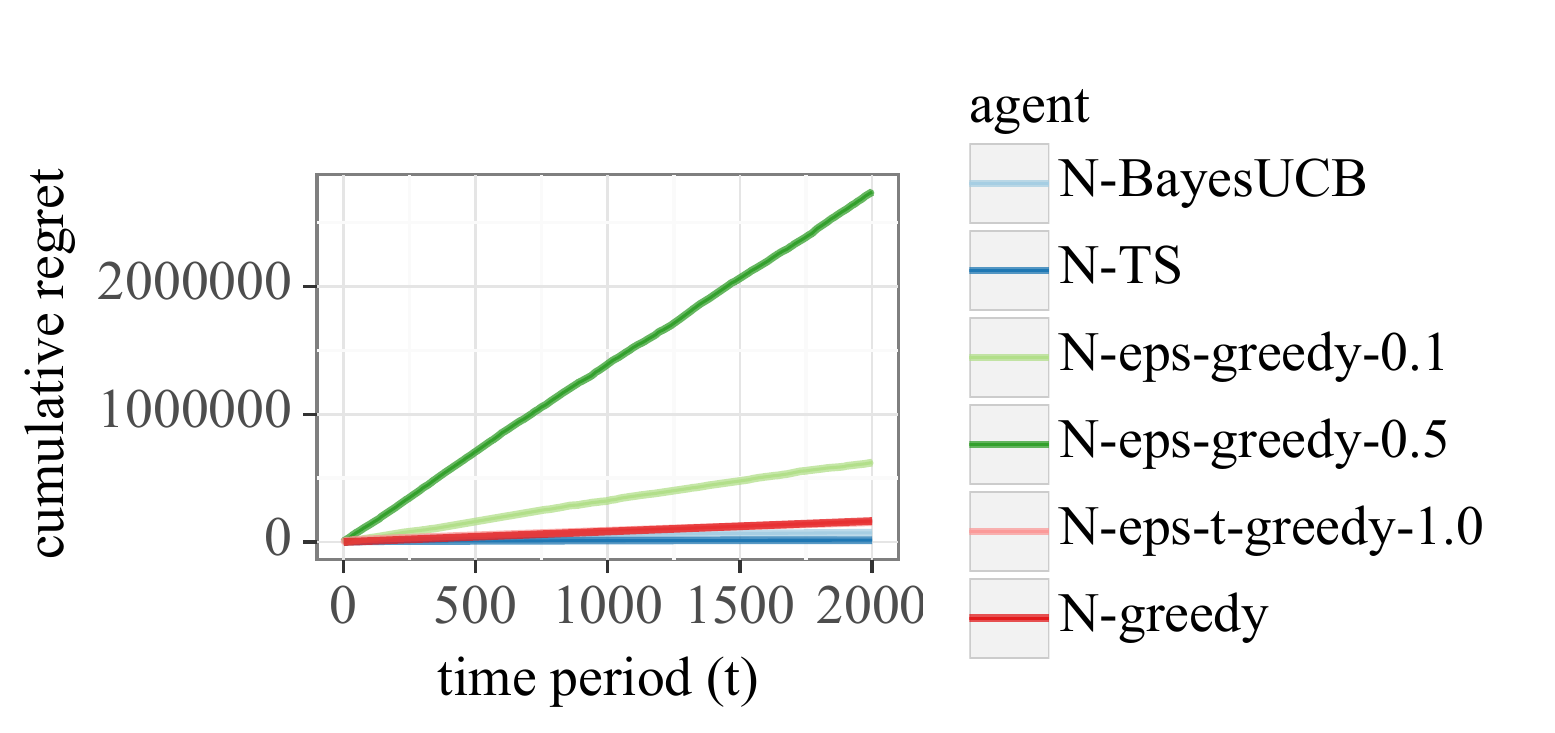}
      \caption{}
      \label{fig:cumulative_regret_luxembourg1_known}
    }
  \end{subfigure}
  &
  \begin{subfigure}[b]{.48\textwidth}
    \centering
    {
      \includegraphics[width=\textwidth]{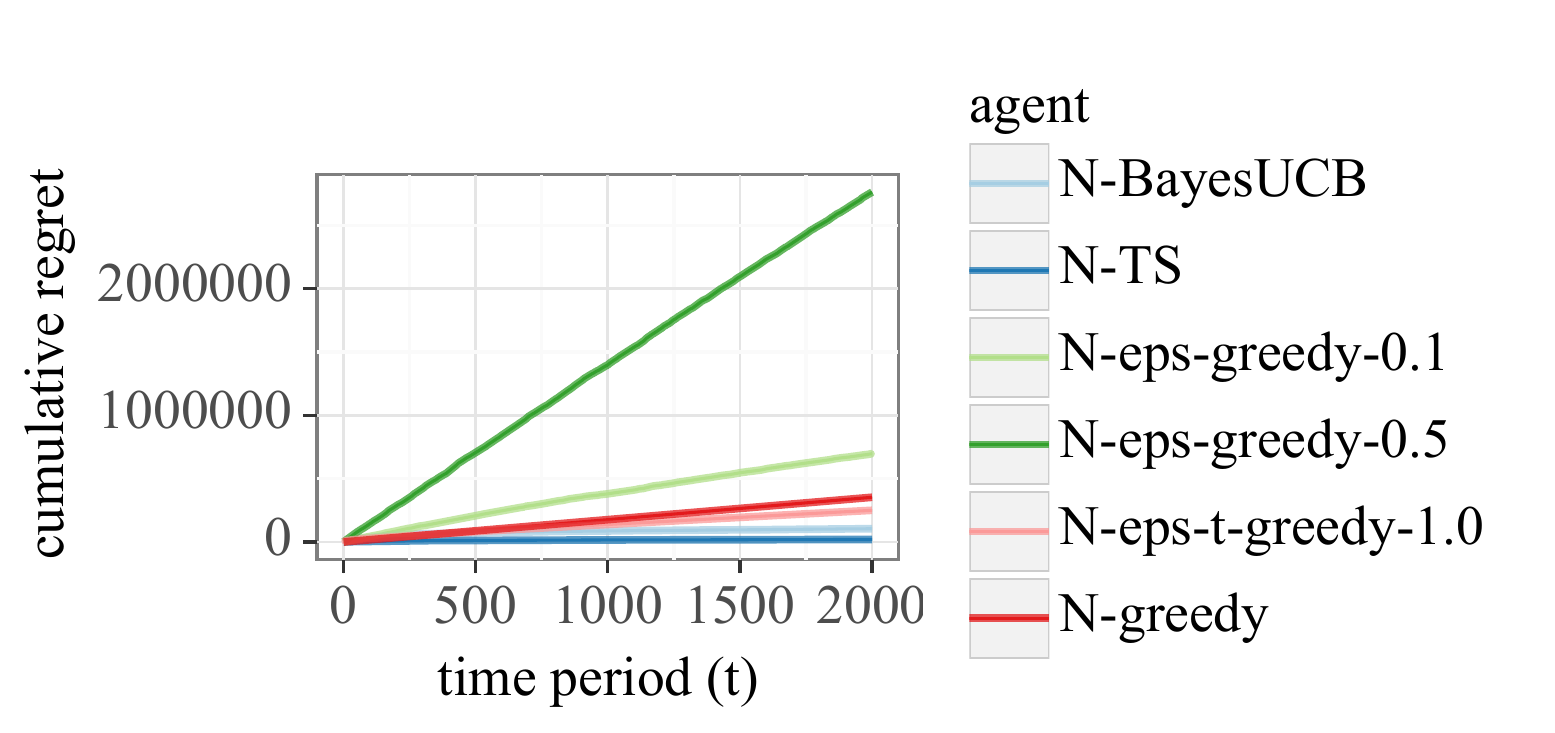}
      \caption{}
      \label{fig:cumulative_regret_luxembourg2_known}
    }
  \end{subfigure}
  \\
  \begin{subfigure}[b]{.48\textwidth}
    \centering
    {
      \includegraphics[width=\textwidth]{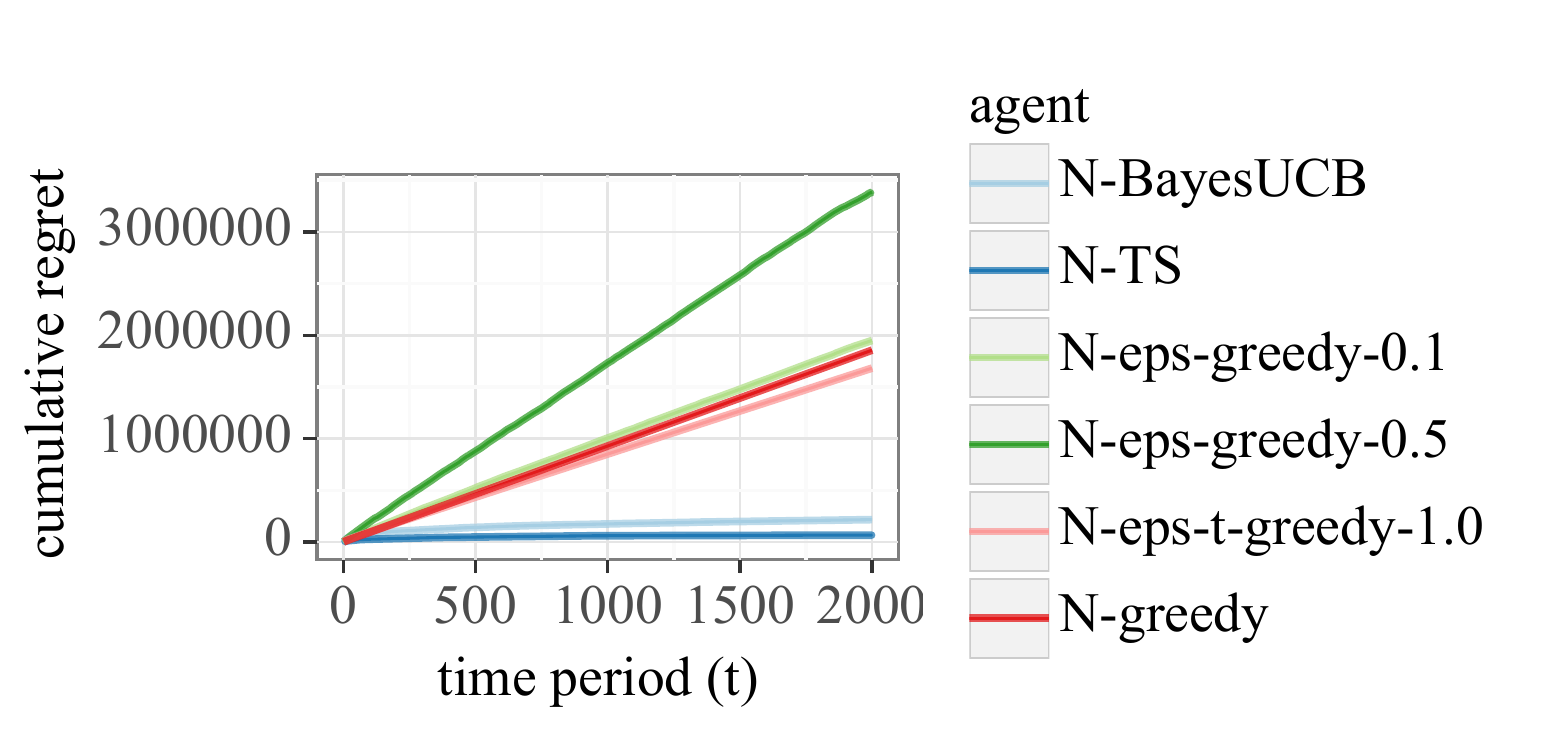}
      \caption{}
      \label{fig:cumulative_regret_monaco1_known}
    }
  \end{subfigure}
  &
  \begin{subfigure}[b]{.48\textwidth}
    \centering
    {
      \includegraphics[width=\textwidth]{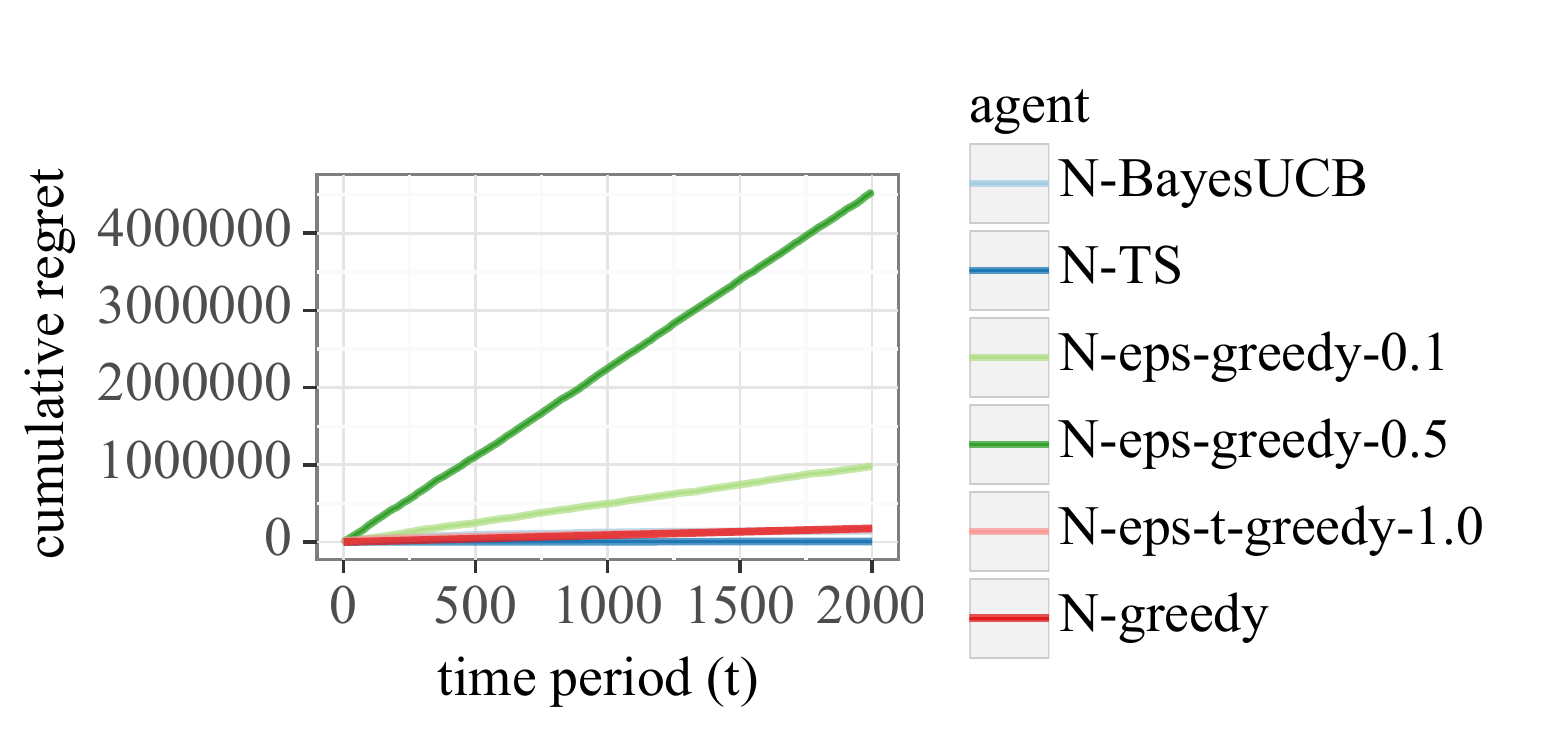}
      \caption{}
      \label{fig:cumulative_regret_turin1_known}
    }
  \end{subfigure}
  \end{tabular}
      \caption{Experimental results on the real-world road networks in the scenario where agents use known priors. For Luxembourg \#1, Luxembourg \#2, Monaco and Turin, respectively, (a), (b), (c) and (d) show the average cumulative regret results for the Thompson Sampling (TS), BayesUCB, $\epsilon_t$-greedy with fixed $\epsilon_t = 0.1$ (eps-greedy-0.1), $\epsilon_t = 0.5$ (eps-greedy-0.5), $\epsilon_t$-greedy with decaying $\epsilon_t$, and probabilistic greedy algorithms, with rectified Gaussian energy consumption models.}
    \label{fig:exp:results_known}
\end{figure*}

\begin{table}
\centering
\begin{tabular}{||c||c|c||c|c||}%
\hline
City&\multicolumn{2}{c||}{Luxembourg \#1}&\multicolumn{2}{c||}{Luxembourg \#2}\\
\hline
Agent&AVG&SD&AVG&SD\\
\hline
N{-}BayesUCB & 72712.10 & 5476.08 & 105222.10 & 7097.73 \\
N{-}TS & \textbf{15062.19} & 3711.03 & \textbf{20747.41} & 3942.83 \\
N{-}eps{-}greedy{-}0.1 & 621259.95 & 93672.80 & 697639.05 & 115938.33 \\
N{-}eps{-}greedy{-}0.5 & 2737852.81 & 97107.13 & 2762055.80 & 160802.01 \\
N{-}eps{-}t{-}greedy{-}1.0 & 155814.50 & 117049.01 & 250879.22 & 224149.22 \\
N{-}greedy & 164903.85 & 119811.93 & 353713.63 & 307249.42 \\
\hline
\hline
City&\multicolumn{2}{c||}{Monaco}&\multicolumn{2}{c||}{Turin}\\
\hline
Agent&AVG&SD&AVG&SD\\
\hline
N{-}BayesUCB & 217256.84 & 70577.34 & 147010.93 & 8855.33 \\
N{-}TS & \textbf{67054.53} & 53218.99 & \textbf{7874.74} & 3559.68 \\
N{-}eps{-}greedy{-}0.1 & 1949999.04 & 2840684.71 & 980895.00 & 134349.00 \\
N{-}eps{-}greedy{-}0.5 & 3385113.34 & 2767510.10 & 4532613.64 & 201149.99 \\
N{-}eps{-}t{-}greedy{-}1.0 & 1682411.76 & 3029371.15 & 167010.98 & 138348.82 \\
N{-}greedy & 1856111.47 & 2954720.83 & 178214.82 & 186056.16 \\
\hline
\end{tabular}
\vspace{4pt}
\caption{Average and standard deviation of regret at $T = 2000$ of agents with known prior distributions. Bold average values indicate the agent with the lowest regret in each scenario.}\label{tbl:standard_known}
\end{table}

 In Section \ref{sec:exp_misspecified} we had realistic unknown energy consumption distributions (fixed across all experiment runs), handled by the agents using misspecified prior distributions. For the second set of experiments, with results shown in Figure \ref{fig:exp:results_known} and Table \ref{tbl:standard_known}, we instead assume that the prior distributions are completely known by the agents. In other words, the environment samples the unknown mean vector $\bm{\theta}^*$ from the prior before all of the agents are applied to the problem instance specified by $\bm{\theta}^*$. Again, the regret results are averaged over 10 runs of each agent, in this setting resulting in an estimate of the Bayesian regret for each agent.

Since we assume that each agent is aware of the true prior distribution in this problem setting, we settle on the rectified Gaussian model of energy consumption for these experiments, with Gaussian prior distributions. As replacements for the Log-Gaussian agents, we increase the number of baselines by implementing a version of $\epsilon$-greedy adapted to combinatorial semi-bandits, based on Algorithm 1 introduced in the supplementary material of \cite{chen2013combinatorial}.

\begin{center}
\begin{algorithm}[th!]
\caption{$\epsilon_t$-greedy for Combinatorial Semi-Bandits}
\label{alg:epsilon_greedy_cmab}
\begin{algorithmic}[1]
\Require Time horizon $T$, prior parameters $\bm{\mu}_0, \bm{\varsigma}_0$, exploration probability $\epsilon_t$ for $t \in [T]$.
\For{$t \leftarrow 1, \dots, T$}
    \State Sample $x \sim \text{Bernoulli}(\epsilon_t)$
    \If{$x = 1$}
        \State Sample an edge $(u_h, u_{h'})$ uniformly from $\mathcal{E}$. 
        \State $\bm{p}_1 \leftarrow$ Shortest path w. r. t. $\bm{\mu}_{t-1}$, between source vertex and $u_h$.
        \State $\bm{p}_2 \leftarrow$ Shortest path w. r. t. $\bm{\mu}_{t-1}$, between $u_{h'}$ and target vertex.
        \State $\bm{a}_t \leftarrow$ Concatenate $\bm{p}_1$ and $\bm{p}_2$.
    \Else
        \State $\bm{a}_t \leftarrow$ Shortest path w. r. t. $\bm{\mu}_{t-1}$, between source and target vertices.
    \EndIf    
    \State Play $\bm{a}_t$, update posterior parameters $\bm{\mu}_t, \bm{\varsigma}_t$ using observed rewards $r_t (\bm{a}_t)$.
\EndFor
\end{algorithmic}
\end{algorithm}
\end{center}

As outlined in Algorithm \ref{alg:epsilon_greedy_cmab}, at each time step $t$ with probability $\epsilon_t$, we select an edge $(u_h, u_{h'}) \in \mathcal{E}$ uniformly at random. We then find the shortest paths with respect to the posterior mean vector, between (1) the source vertex of the problem instance and $u_h$, and (2) $u_{h'}$ and the target vertex. The resulting concatenated path, including the edge $(u_h, u_{h'})$, is used to explore the road network graph. With probability $1-\epsilon_t$, we instead greedily select the shortest path between the source and target vertices, exploiting the current posterior mean estimates.

We evaluate agents using constant values of $\epsilon_t$ ($0.1$ and $0.5$), as well as an agent $\epsilon_t$ decaying in $t$ (with $\epsilon_t = \frac{1}{t}$). We motivate the latter with Theorem 4 in the supplementary material of \cite{chen2013combinatorial}, where the authors show a sub-linear upper bound on the expected regret of their $\epsilon_t$-greedy algorithm, with $\epsilon_t$ in the order of $\frac{1}{t}$ (with an additional constant factor derived from information about the problem instance).

As shown in Figures \ref{fig:cumulative_regret_luxembourg1_known}, \ref{fig:cumulative_regret_luxembourg2_known}, \ref{fig:cumulative_regret_monaco1_known} and \ref{fig:cumulative_regret_turin1_known}, the results from the experiments with the TS, BayesUCB and (pure) greedy agents closely match the corresponding experiments in the misspecified prior problem setting of the previous section, while $\epsilon_t$-greedy with decaying $\epsilon_t$ has comparable performance to the greedy agent. The $\epsilon_t$-greedy agents with constant $\epsilon_t$ perform consistently worse than the other agents. Also supported by Table \ref{tbl:standard_known}, the regret of the TS agent still saturates rapidly and achieves the best average regret out of the evaluated agents for all cities.

\subsubsection{Networks with correlated edge weights} \label{sec:exp_correlated}

\begin{figure*}%[!t]
  \centering
  \begin{tabular}{cc}
  \begin{subfigure}[b]{.48\textwidth}
    \centering
    {
      \includegraphics[width=\textwidth]{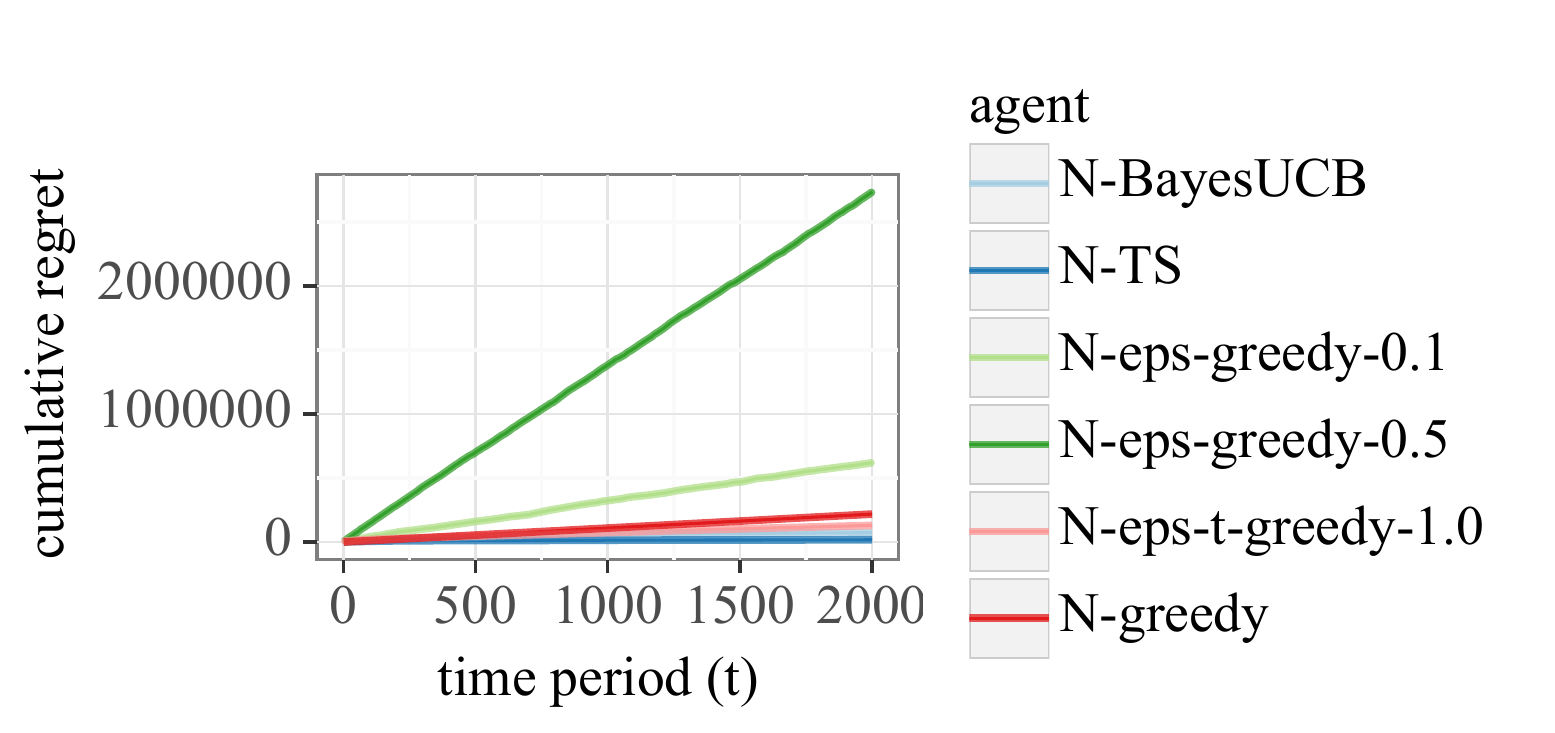}
      \caption{}
      \label{fig:cumulative_regret_luxembourg1_correlated}
    }
  \end{subfigure}
  &
  \begin{subfigure}[b]{.48\textwidth}
    \centering
    {
      \includegraphics[width=\textwidth]{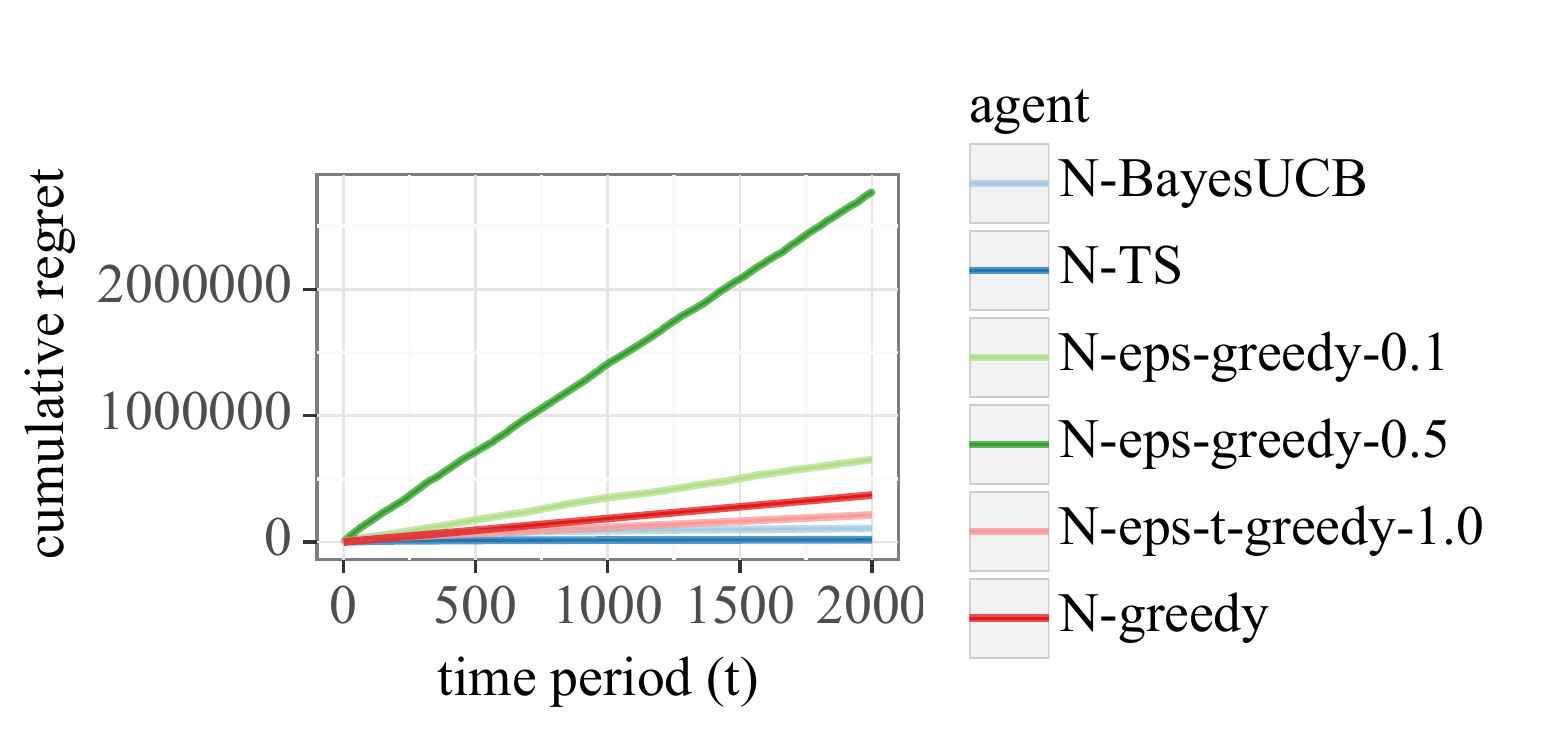}
      \caption{}
      \label{fig:cumulative_regret_luxembourg2_correlated}
    }
  \end{subfigure}
  \\
  \begin{subfigure}[b]{.48\textwidth}
    \centering
    {
      \includegraphics[width=\textwidth]{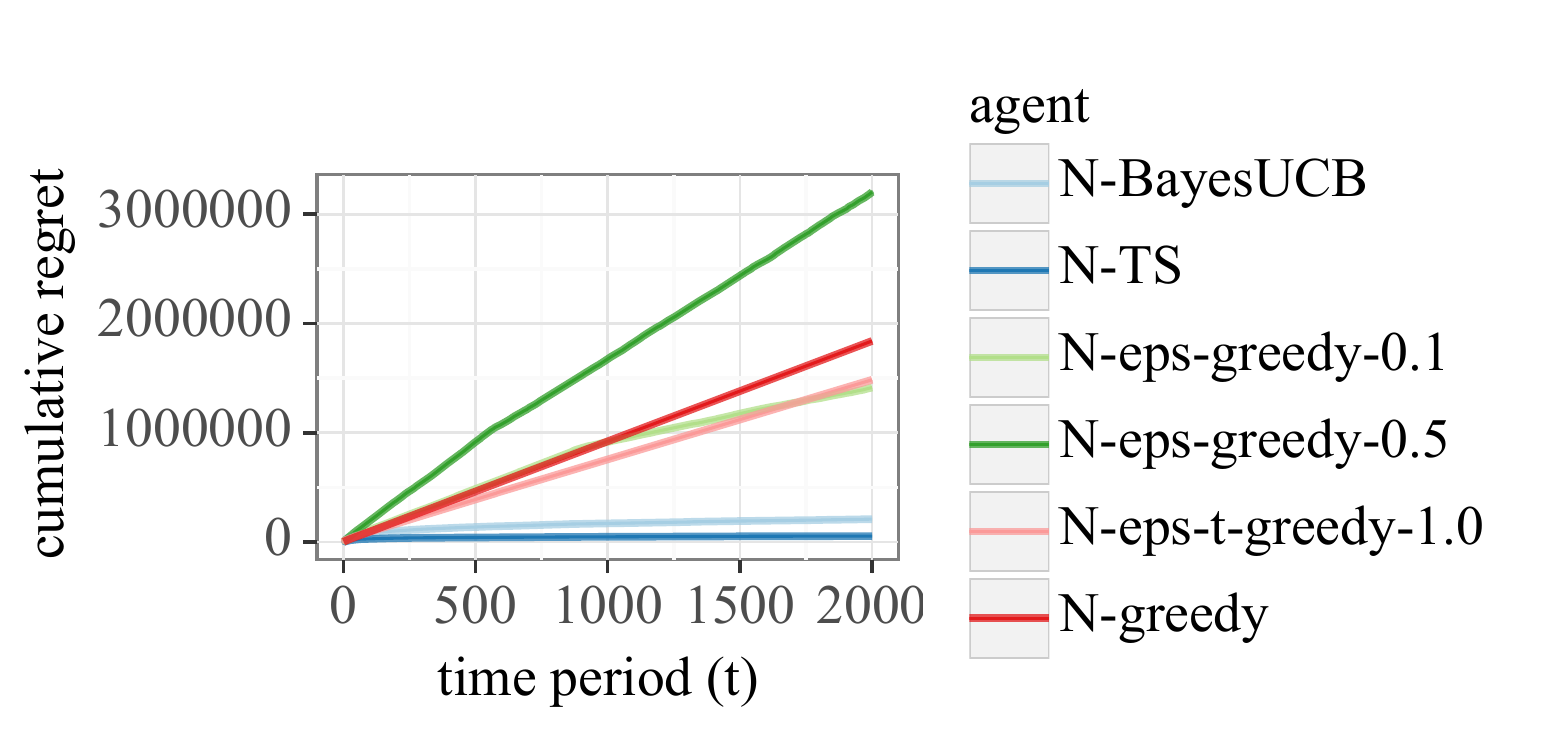}
      \caption{}
      \label{fig:cumulative_regret_monaco1_correlated}
    }
  \end{subfigure}
  &
  \begin{subfigure}[b]{.48\textwidth}
    \centering
    {
      \includegraphics[width=\textwidth]{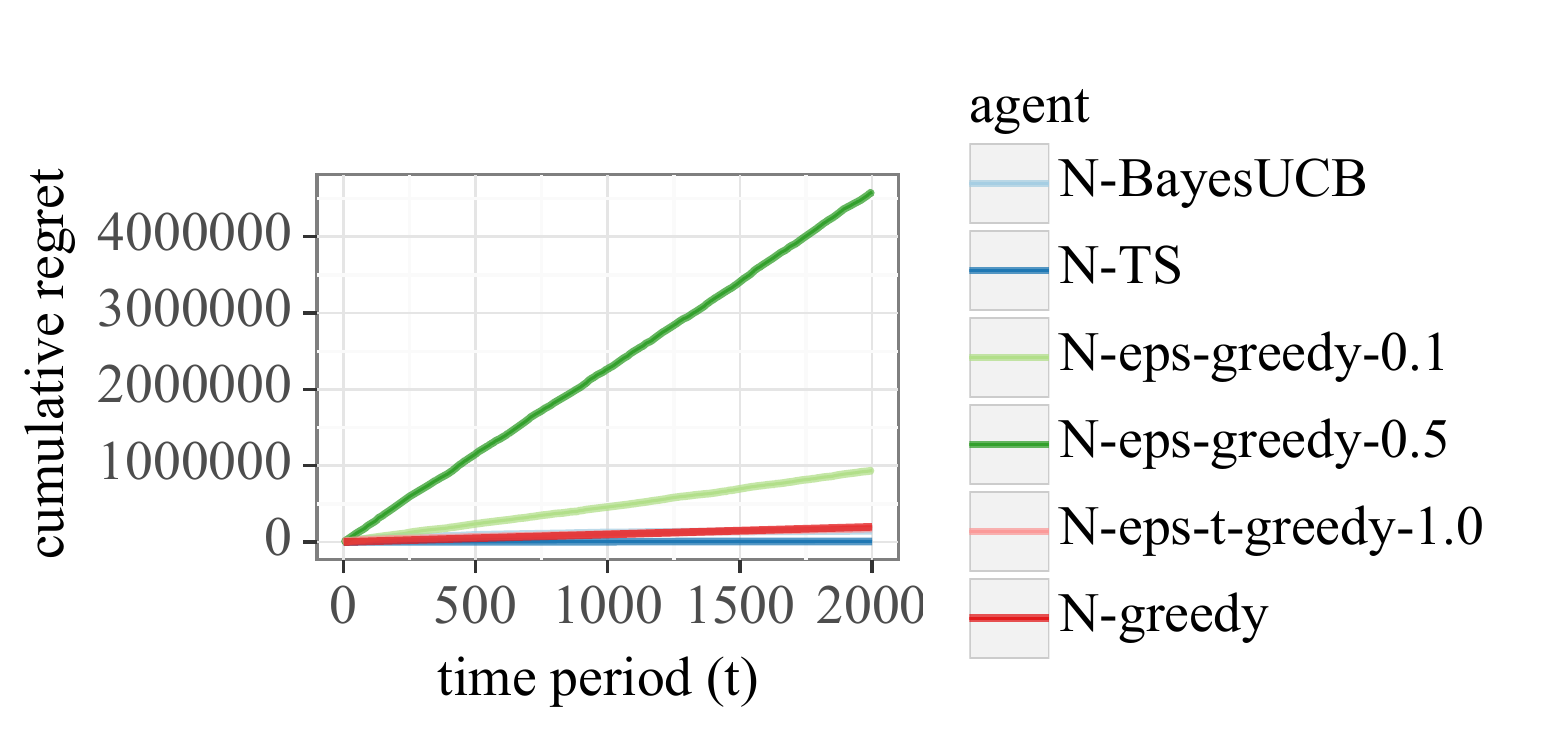}
      \caption{}
      \label{fig:cumulative_regret_turin1_correlated}
    }
  \end{subfigure}
  \end{tabular}
      \caption{Experimental results on the real-world road networks in the scenario where there is correlation between edges in the environments. For Luxembourg \#1, Luxembourg \#2, Monaco and Turin, respectively, (a), (b), (c) and (d) show the average cumulative regret results for the Thompson Sampling (TS), BayesUCB, $\epsilon_t$-greedy with fixed $\epsilon_t = 0.1$ (eps-greedy-0.1), $\epsilon_t = 0.5$ (eps-greedy-0.5), $\epsilon_t$-greedy with decaying $\epsilon_t$, and probabilistic greedy algorithms, with rectified Gaussian energy consumption models.}
    \label{fig:exp:results_correlated}
\end{figure*}

\begin{table}
\centering
\begin{tabular}{||c||c|c||c|c||}%
\hline
City&\multicolumn{2}{c||}{Luxembourg \#1}&\multicolumn{2}{c||}{Luxembourg \#2}\\
\hline
Agent&AVG&SD&AVG&SD\\
\hline
N{-}BayesUCB & 73439.29 & 7133.48 & 108704.02 & 9960.32 \\
N{-}TS & \textbf{18313.43} & 2235.25 & \textbf{18351.64} & 4221.91 \\
N{-}eps{-}greedy{-}0.1 & 619592.02 & 100767.70 & 652295.17 & 115089.44 \\
N{-}eps{-}greedy{-}0.5 & 2732324.05 & 160432.90 & 2768669.58 & 141189.83 \\
N{-}eps{-}t{-}greedy{-}1.0 & 130141.09 & 85763.04 & 214824.01 & 173082.45 \\
N{-}greedy & 218541.56 & 208088.80 & 371926.16 & 326922.34 \\
\hline
\hline
City&\multicolumn{2}{c||}{Monaco}&\multicolumn{2}{c||}{Turin}\\
\hline
Agent&AVG&SD&AVG&SD\\
\hline
N{-}BayesUCB & 209794.25 & 57753.94 & 148316.20 & 8128.75 \\
N{-}TS & \textbf{54797.54} & 22390.01 & \textbf{8809.96} & 4657.47 \\
N{-}eps{-}greedy{-}0.1 & 1414164.08 & 1406368.66 & 936641.81 & 91197.23 \\
N{-}eps{-}greedy{-}0.5 & 3202314.86 & 1601363.76 & 4584797.59 & 162979.92 \\
N{-}eps{-}t{-}greedy{-}1.0 & 1485022.56 & 2844251.84 & 178088.28 & 142590.73 \\
N{-}greedy & 1840158.61 & 3102682.97 & 199636.06 & 186795.70 \\
\hline
\end{tabular}
\vspace{4pt}
\caption{Average and standard deviation of regret at $T = 2000$ of agents, where there is correlation between edges in the environments. Bold average values indicate the agent with the lowest regret in each scenario.}\label{tbl:correlated_known}
\end{table}

To demonstrate that the proposed framework performs well even when a few environment assumptions are relaxed, we run an additional set of experiments in a variation of the setting described in Section \ref{sec:exp_known_prior}, with results shown in Figure \ref{fig:exp:results_correlated} and Table \ref{tbl:correlated_known}. Whereas in the previous sections the stochastic weights of all edges are assumed to be mutually independent, we now introduce correlation between edge weights. An example of this in real-world road networks can be that traffic congestion on one road segment is likely to affect nearby road segments as well. 

As in the previous section, a mean vector $\bm{\theta}^*$ unknown to the agents is generated by the environment, where each element is sampled independently from the (Gaussian) prior distribution of each edge in the road network. Subsequently, we randomly assign all edges in $\mathcal{E}$ to a set of $\vert \mathcal{E} \vert / 2$ pairs of edges. We let the energy consumption of the individual edges in each such pair of edges $(e, e') \in \mathcal{E} \times \mathcal{E}$ be perfectly correlated, but we define the marginal distributions according to the model in Section \ref{sec:rectified_gaussian_model}. In each time step, we jointly sample the energy consumption for each pair $(e, e')$ from a two-dimensional distribution with mean vector $\bm{\theta}^*_{(e, e')}$ and covariance matrix $\Sigma_{(e, e')}$, defined as

\begin{align*}
    \bm{\theta}^*_{(e, e')} &= \begin{bmatrix}
           \theta^*_{e} \\
           \theta^*_{e'}
           \end{bmatrix}, \;\;\;
           \Sigma_{(e, e')} = \begin{bmatrix}
           \sigma_e^2 & \sigma_e \sigma_{e'} \\
           \sigma_e \sigma_{e'} & \sigma_{e'}^2
           \end{bmatrix} .
\end{align*}

Beyond the generation of correlated energy consumption by the environment, the experiments are set up exactly as in Section \ref{sec:exp_known_prior}. The agents are assumed to be unaware of the correlation, and only attempt to estimate the parameters of the marginal distributions. As shown in Figures \ref{fig:cumulative_regret_luxembourg1_correlated}, \ref{fig:cumulative_regret_luxembourg2_correlated}, \ref{fig:cumulative_regret_monaco1_correlated} and \ref{fig:cumulative_regret_turin1_correlated}, as well as in Table \ref{tbl:correlated_known}, when compared with results in the previous section, the performance of the agents is not noticeably affected by the presence of correlation. 

\subsection{Synthetic Networks} \label{sec:exp:sn}

%%%%%%%%%%%%%%%%%%%%%%%%%%%%%%%%%%%%%%%%%%%%%%%%%%%%%%%%%%%%%%%%%%%%
% Grouping first 3 plots in a panel

\begin{figure*}[!t]
  \centering
  \includegraphics[trim={0pt 15pt 0pt 20pt}, width=0.5\textwidth]{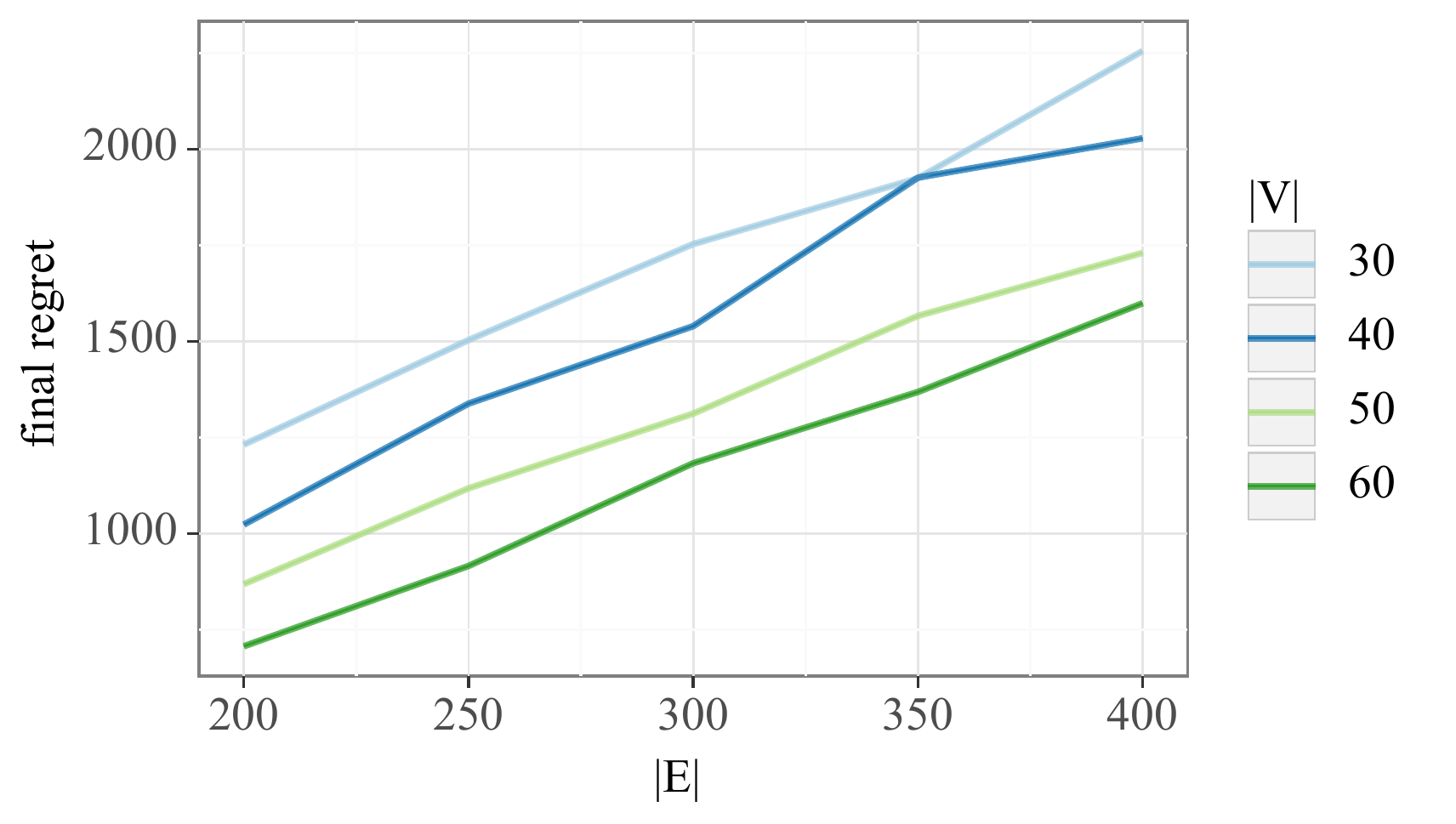}
  \caption{Final cumulative regret ($T = 2000$) on synthetic networks as a function of $\cardE$.}
  \label{fig:cumulative_regret_edges}
\end{figure*}

%%%%%%%%%%%%%%%%%%%%%%%%%%%%%%%%%%%%%%%%%%%%%%%%%%%%%%%%%%%%%%%%%%%%

%\paragraph{Settings.} 
In order to evaluate the regret bound in Proposition \ref{prop:graph_regret_bound_ts}, we design synthetic directed acyclic network instances $\mathcal{G}(\mathcal{V},\mathcal{E},\bm{w})$ according to a specified number of vertices $n$ and number of edges $o$ (with the constraint that $n - 1 \leq o \leq n (n - 1) / 2$). We start the procedure by adding $n$ vertices $u_1, \dots, u_{n}$ to $\mathcal{V}$. Then for each $h \in [1,n-1]$ we add an edge $(u_h,u_{h+1})$ to $\mathcal{E}$. This ensures that the network contains a path with all vertices in $\mathcal{V}$. Finally, we add $o-n$ edges $(u_h,u_{h'})$ uniformly at random to $\mathcal{E}$, such that $h \neq h'$, $h+1 \neq h'$ and $h < h'$.

Since these networks are synthetic, instead of modeling probabilistic energy consumption, we design instances where it is difficult for an exploration algorithm to find the path with the lowest expected cost. Given a synthetic network $\mathcal{G}$ generated according to the aforementioned procedure, we select $\bm{p} = \langle u_1, \dots, u_n \rangle$ to be the optimal path. In other words, $\bm{p}$ contains every vertex $u \in \mathcal{V}$. The reward distribution for each edge $e$ in $\bm{p}$ is chosen to be $\mathcal{N}(-\Tilde{E}_e | \theta^*_e, \sigma_e^2)$ with $\theta^*_e = -10$ and $\sigma_e^2 = 4$. For $(u_h, u_{h'}) \in E$ where $(u_h, u_{h'}) \notin \bm{p}$, we set $\theta^*_e = -11(h'-h)$, where $h'-h$ is the number of vertices skipped by the shortcut. This guarantees that no matter the size of the network and the number of edges that form shortcuts between vertices in $\bm{p}$, $\bm{p}$ will always have a lower expected cost than any other path in $\mathcal{G}$.

For the agent prior $\mathcal{N}(\theta^*_e | \mu_{e,0}, \varsigma_{e,0}^2)$, we set $\mu_{e,0} = -11(h'-h)$ and $\varsigma_{e,0}^2 = 8$. This choice of prior mean implies according to our prior beliefs, every path from the source $u_1$ to the target $u_n$ will initially have the same estimated expected cost.

We run the synthetic network experiment with $T = 2000$ time steps, varying the number of vertices $\cardV \in \{30, 40, 50, 60\}$ and edges $\cardE \in \{200, 250, 300, 350, 400\}$. In Figure \ref{fig:cumulative_regret_edges}, each plot represents the cumulative regret at $T = 2000$ for a fixed $\cardV$, as a function of $\cardE$. We observe that the regret increases no more than linearly with the number of edges, which is consistent with the theoretical regret bound in Corollary \ref{cor:graph_regret_bound_ts2}. %

\section{Conclusion}
We developed a Bayesian online learning framework for the problem of energy efficient navigation of electric vehicles. Our Bayesian model assumes a rectified Gaussian or Log-Gaussian energy model. To learn  the unknown parameters of the model, we adapted exploration methods such as Thompson Sampling and BayesUCB within the online learning framework. We extended the framework to multi-agent and batched feedback settings, and established theoretical regret bounds. Finally, we demonstrated the performance of the framework with several real-world and synthetic experiments. %

\section*{Acknowledgement}
Niklas {\AA}kerblom is a PhD student employed by Volvo Car Corporation. This work is funded by the Strategic Vehicle Research and Innovation Programme (FFI) of Sweden, through the project EENE (reference number: 2018-01937).

\appendix
\section{Notation}\label{sec:notation}
\setcounter{table}{0}
\renewcommand{\thetable}{A\arabic{table}}

The notation used throughout the paper is summarized below, in Table \ref{tbl:notation}. Note that since each each edge $e \in \mathcal{E}$ corresponds to a base arm $i \in \mathcal{A}$, these are used interchangeably as subscript indices to various variables.

\begin{longtable}{||c|c||}%t
\hline%
Notation&Description\\%
\hline%
\endhead
\hline%
\caption{Summary of the notation used throughout the paper.}\\ %
\multicolumn{2}{r}{\footnotesize Continued on the next page}%?
\endfoot 
\caption{Summary of the notation used throughout the paper.} \label{tbl:notation}
\endlastfoot
$\bm{a}$ & A super-arm\\
$\bm{a}_t$ & Super-arm selected at time $t$\\
$\bm{a}^*$ & Optimal super-arm\\
$a$ & An arm\\
$a_t$ & Arm selected at time $t$\\
$b$ & Batch index (in batched feedback setting)\\
$\bm{b}$ & A super-arm (alternative)\\
$\bm{c}$ & A cycle (path)\\
$d$ & Number of base arms\\
$e$ & An edge\\
$g$ & Gravitational acceleration ($\text{m}/\text{s}^2$)\\
$h$ & Vertex index\\
$h'$ & Vertex index (alternative)\\
$i$ & A base arm\\
$j$ & A base arm (alternative)\\
$k$ & Agent index (in multi-agent setting)\\
$l_e$ & Length of edge $e$ (m)\\
$m$ & Vehicle mass (kg)\\
$n$ & Final (vertex) index (of, e.g., a path)\\
$o$ & Final number of edges in synthetic network setting\\
$\bm{p}$ & A path (connected sequence of vertices / edges)\\
$s$ & A time step / round (alternative)\\
$t$ & A time step / round\\
$t_b$ & Last time step / round of batch $b$\\
$u$ & A vertex\\
$v$ & Vehicle speed (m/s)\\
$\bm{w}_t$ & Edge weight vector at time $t$\\
$w_{e, t}$ & Weight of edge $e$ at time $t$\\
$z_e$ & Rectified (Gaussian) energy consumption edge $e$\\
$A$ & Front surface area of vehicle ($\text{m}^2$)\\
$B$ & Number of batches\\
$C_{r}$ & Rolling resistance coefficient of edge $e$\\
$C_d$ & Air drag coefficient of vehicle\\
$E_e$ & Approximated energy consumption of edge $e$ (Wh)\\
$\Tilde{E}_e$ & Stochastic energy consumption of edge $e$ (Wh)\\
$H_t$ & History (random) of actions and rewards until time $t$\\
$H$ & A realized (fixed) history of actions and rewards\\
$K$ & Number of agents in the multi-agent setting\\
$M$ & A Markov Decision Process (MDP)\\
$T$ & Time horizon\\
$\alpha_e$ & Inclination angle of edge $e$ (radians)\\
$\beta$ & Probability threshold parameter of quantile function\\
$\Delta_t$ & Instant regret (suboptimality gap) at time $t$\\
$\epsilon$ & Exploration probability of $\epsilon$-greedy algorithm\\
$\epsilon_t$ & Exploration probability of $\epsilon_t$-greedy algorithm at time $t$\\
$\eta$ & Powertrain efficiency of vehicle\\
$\eta^+$ & Powertrain efficiency of vehicle during traction\\
$\eta^-$ & Powertrain efficiency of vehicle during braking\\
$\bm{\theta}$ & A mean vector\\
$\hat{\bm{\theta}}_t$ & Average reward vector until time $t$\\
$\tilde{\bm{\theta}}$ & Sampled mean reward vector\\
$\bm{\theta}^*$ & True mean reward vector\\
$\bm{\theta}^*_{(i, j)}$ & True mean reward vector of correlated base arms $i$ and $j$\\
$\hat{\theta}_{i,t}$ & Average reward of base arm $i$ until time $t$\\
$\tilde{\theta}_i$ & Sampled mean reward of base arm $i$\\
$\theta_i^*$ & True mean reward of base arm $i$\\
$\vartheta$ & Factor of the mean to calculate prior standard deviation\\
$\kappa_t$ & Feedback delay of arm selected at time $t$\\
$\lambda$ & A distribution\\
$\bm{\mu}_{0}$ & Prior mean vector\\
$\bm{\mu}_{t}$ & Posterior mean vector at time $t$\\
$\mu_{i,0}$ & Prior mean of base arm $i$\\
$\mu_{i,t}$ & Posterior mean of base arm $i$ at time $t$\\
$\bar{\nu}_{i, x}$ & Average reward of base arm $i$ over first $x$ plays\\
$\rho$ & Air density ($\text{kg} / \text{m}^3$)\\
$\sigma_i^2$ & Noise variance of base arm $i$\\
$\Sigma_{(i, j)}$ & Covariance matrix of correlated base arms $i$ and $j$\\
$\bm{\varsigma}_{0}$ & Prior standard deviation vector\\
$\bm{\varsigma}_{t}$ & Posterior standard deviation vector at time $t$\\
$\varsigma_{i, 0}^2$ & Prior variance of base arm $i$\\
$\varsigma_{i, t}^2$ & Posterior variance of base arm $i$ at time $t$\\
$\tau$ & Episode length of reinforcement learning problem\\
$\varphi$ & Factor of the mean to calculate noise standard deviation\\
$\psi_i$ & Gaussian prior mean of base arm $i$ \\
$\mathcal{A}$ & Set of base arms\\
$\mathcal{D}_t$ & Set of delayed rewards at time $t$\\
$\mathcal{E}$ & Set of edges\\
$\mathcal{G}(\mathcal{V}, \mathcal{E}, \bm{w})$ & A (weighted) graph with vertices $\mathcal{V}$, edges $\mathcal{E}$, and weights $\bm{w}$\\
$\mathcal{I}$ & Set of super-arms\\
$\mathcal{M}$ & Set of Markov Decision Processes (MDPs)\\
$\mathcal{P}$ & Set of paths\\
$\mathcal{V}$ & Set of vertices\\
$\mathbb{R}$ & Set of real numbers\\
$\mathbb{R}^+$ & Set of non-negative real numbers\\
$\text{BayesRegret}(T)$ & Bayesian regret until horizon $T$\\
$\text{BayesRegret}_k (T)$ & Bayesian regret of agent $k$ until horizon $T$\\
$\text{Bernoulli}(\cdot)$ & Bernoulli distribution\\
$\mathbb{E}[\cdot]$ & Expected value of random variable\\
$f_{\bm{\theta}} (\bm{a})$ & Expected reward of super-arm $\bm{a}$, given the mean vector $\bm{\theta}$\\
\multirow{2}{*}{$f^R_{\bm{\theta}} (\bm{a})$} & Expected reward of super-arm $\bm{a}$, given the mean vector $\bm{\theta}$, \\ & under rectified Gaussian base arm feedback\\
$L(\bm{a}, H)$ & Lower confidence bound of super-arm $\bm{a}$ given history $H$\\
$\mathcal{LN}(\cdot)$ & Log-Gaussian distribution\\
$\mathbf{Mode}[\cdot]$ & Mode of random variable\\
$\mathcal{N}(\cdot)$ & Gaussian distribution\\
$\mathcal{N}^R(\cdot)$ & Rectified Gaussian distribution\\
$N_t (i)$ & Number of plays of base arm $i$ until time $t$\\
$\mathcal{O}(\cdot)$ & Order of a function\\
$\tilde{\mathcal{O}}(\cdot)$ & Order of a function (excluding polylogarithmic factors)\\
$P(\cdot)$ & Probability distribution of random variable\\
$\text{Pr}\{\cdot\}$ & Probability of event\\
$Q(\beta, \lambda)$ & Quantile function of distribution $\lambda$ with probability threshold $\beta$\\
$\text{Queue}[\bm{a}]$ & Delayed feedback queue of super-arm $\bm{a}$\\
$r_t (\bm{a})$ & Reward of super-arm $\bm{a}$ at time $t$\\
$\text{Regret}(T)$ & Frequentist regret until horizon $T$\\
$\text{Regret}_k (T)$ & Frequentist regret of agent $k$ until horizon $T$\\
$U(\bm{a}, H)$ & Upper confidence bound of super-arm $\bm{a}$ given history $H$\\
$\mathbf{Var}[\cdot]$ & Variance of random variable\\
$\phi(x)$ & Standard Gaussian probability density function (PDF)\\
$\Phi(x)$ & Standard Gaussian cumulative distribution function (CDF)\\
\hline%
\end{longtable}

%\vskip 0.2in
%\newpage

\bibliography{main}

\end{document}